\documentclass[3p,12pt]{elsarticle}

\usepackage{amsmath,amssymb,amsthm,mathtools}
\usepackage{stmaryrd}
\usepackage{latexsym}
\usepackage{bbm}
\usepackage{subfig}
\usepackage{multirow}
\usepackage{rotating}
\usepackage[section]{placeins} 
\usepackage{lineno}
\usepackage{hyperref}
\usepackage{url}
\usepackage{xcolor}
\definecolor{newcolor}{rgb}{.8,.349,.1}

\usepackage{algorithmic}
\usepackage{algorithm}

\journal{Fuzzy Sets and Systems}

\newcommand{\new}{}

\newcommand{\reels}{\mathbb{R}}

\newcommand{\esp}{\mathbb{E}}
  
\renewcommand{\Pr}{\mathbb{P}} 
\newcommand{\bbX}{{\overline{X}}}

\def\cov{\text{Cov}}
\def\var{\text{Var}}
\newcommand{\one}{\mathbf{1}}
\newcommand{\height}{\textsf{hgt}}

\newcommand{\calS}{{\cal S}}

\newcommand{\calN}{{\cal N}}

\newcommand{\calX}{{\cal X}}

\newcommand{\tF}{{\widetilde{F}}}

\newcommand{\tG}{{\widetilde{G}}}

\newcommand{\tA}{{\widetilde{A}}}
\newcommand{\tB}{{\widetilde{B}}}
\newcommand{\tC}{{\widetilde{C}}}
\newcommand{\tP}{{\widetilde{P}}}
\newcommand{\tN}{{\widetilde{N}}}

\newcommand{\tX}{{\widetilde{X}}}

\newcommand{\oh}{{\overline{h}}}

\newcommand{\tTheta}{{\widetilde{\Theta}}}
\newcommand{\ttheta}{{\widetilde{\theta}}}
\newcommand{\tmu}{{\widetilde{\mu}}}
\newcommand{\tbmu}{{\widetilde{\bmu}}}
\newcommand{\tsigma}{{\widetilde{\sigma}}}

\newcommand{\tbSigma}{{\widetilde{\bSigma}}}

\newcommand{\obH}{{\overline{\bH}}}
\newcommand{\tY}{{\widetilde{Y}}}

\def\bmu{\boldsymbol{\mu}}

\def\bzero{\boldsymbol{0}}

\def\btheta{\boldsymbol{\theta}}
\def\bSigma{\boldsymbol{\Sigma}}

\def\bx{{\boldsymbol{x}}}

\def\bm{{\boldsymbol{m}}}

\def\bI{{\boldsymbol{I}}}
\def\bH{{\boldsymbol{H}}}

\def\bX{{\boldsymbol{X}}}

\def\bM{{\boldsymbol{M}}}

\def\bC{{\boldsymbol{C}}}

\def\bB{{\boldsymbol{B}}}
\def\bA{{\boldsymbol{A}}}

\def\cut#1#2{{}^#1#2}
\newcommand{\fracpar}[2]{\left(\frac{#1}{#2}\right)}

\def\GFN{\textsf{GFN}}
\def\GFV{\textsf{GFV}}

\def\block#1#2#3#4{
\begin{pmatrix}
#1& #2\\
#3 &#4
\end{pmatrix}}

\newcommand{\deriv}[2]{\frac{\partial #1}{\partial #2}}

\newcommand{\bi}{\begin{itemize}}
\newcommand{\ei}{\end{itemize}}
\newcommand{\be}{\begin{enumerate}}
\newcommand{\ee}{\end{enumerate}}
\newcommand{\bd}{\begin{description}}
\newcommand{\ed}{\end{description}}

\newtheorem{Prop}{Proposition}  

\newtheorem{Ex}{Example}

\newtheorem{Lem}{Lemma}
\newtheorem{Cor}{Corollary}
\newtheorem{Def}{Definition}

\begin{document}
\begin{frontmatter}
\title{
Reasoning with fuzzy and uncertain evidence using epistemic random fuzzy sets: 
general framework and practical models\footnote{This paper was published in \emph{Fuzzy Sets and Systems}, 453:1--36, 2023. This version corrects an error in Equation \eqref{eq:belint}.}}

\author[utc,iuf]{Thierry Den{\oe}ux}
\ead{Thierry.Denoeux@utc.fr}

\address[utc]{Universit\'e de technologie de Compi\`egne, CNRS\\
UMR 7253 Heudiasyc, Compi\`egne, France}
\address[iuf]{Institut universitaire de France, Paris, France}

\begin{abstract}
We introduce a general theory of epistemic random fuzzy sets for reasoning with fuzzy or crisp evidence. This framework generalizes both the Dempster-Shafer theory of belief functions, and possibility theory. Independent epistemic random fuzzy sets are combined by the generalized product-intersection rule, which extends both Dempster's rule for combining belief functions, and the product  conjunctive combination of possibility distributions. We introduce Gaussian random fuzzy numbers and their multi-dimensional extensions, Gaussian random fuzzy vectors, as practical models for quantifying uncertainty about scalar or vector quantities. Closed-form expressions for the  combination, projection and vacuous extension of Gaussian random fuzzy numbers and vectors are derived.
\end{abstract}


\end{frontmatter}

\section{Introduction}

The Dempster-Shafer (DS) theory of belief functions \cite{shafer76} and possibility theory \cite{zadeh78} were introduced independently in the  late 1970's as non-probabilistic frameworks for reasoning with uncertainty \cite{denoeux20a,denoeux20b}. The former approach is based on the idea of representing elementary pieces of evidence as completely monotone capacities, or \emph{belief functions}, and combining them using an operator known as the product-intersection rule or Dempster's rule. As probability measures are special belief functions, and Dempster's rule extends Bayesian conditioning, DS can be seen as an extension of Bayesian probability theory, particularly suitable to reasoning with severe uncertainty. There is also a strong relation between DS theory and the theory of random sets \cite{molchanov05}: specifically, any random set induces a belief function and, conversely, any belief function can be seen as being induced by some random set \cite{nguyen78}. In  DS theory,  a random set underlying a belief function does not represent a random mechanism for generating sets of outcomes, but  the imprecise meanings of a piece of evidence under different interpretations  with known probabilities \cite{shafer81}. To avoid confusion, we  use the term \emph{epistemic random set} for random sets representing evidence in DS theory.

In contrast, possibility theory originates from the theory of fuzzy sets \cite{zadeh65}. In this approach, a fuzzy statement about the variable of interest, seen as a flexible constraint on its precise but unknown value in some domain $\Theta$, induces a possibility measure and a dual necessity measure on $\Theta$. Interestingly, a necessity measure is a belief function, and the dual possibility measure is the corresponding plausibility function, but the converse is not true (a belief function is not, in general, a necessity measure). For this reason, possibility theory has sometimes been presented as  ``a special branch of evidence theory'' (another name for DS theory) \cite[page 187]{klir95}. However, combining two necessity measures by Dempster's rule yields a belief function that is no longer a necessity measure: this combination rule is, thus, not compatible with possibilistic reasoning. In contrast, possibility theory has its own conjunctive combination operators based on triangular norms (or t-norms) \cite{dubois99}. Possibility and DS theory are, thus, two distinct models of uncertain reasoning based on related knowledge representation languages but different information processing mechanisms. 

In a companion paper \cite{denoeux21a}, we have revisited Zadeh's notion of ``evidence of the second kind'', defined as a pair $(X,\Pi_{(Y\mid X)})$ in which $X$ is a discrete random variable on a set $\Omega$ and  $\Pi_{(Y\mid X)}$ a collection of conditional possibility distributions of a variable $Y$ given $X=x$, for all $x\in \Omega$.  If random variable $X$ is constant, we get a unique possibility distribution for variable $Y$; if the conditional possibility distributions $\Pi_{(Y\mid X)}$ take values in $\{0,1\}$, then the pair $(X,\Pi_{(Y\mid X)})$ defines a random set equivalent to a DS mass function. The mappings associating, to each event, its expected necessity and its expected possibility are, respectively, belief and plausibility functions. In this framework, a possibility distribution thus represents certain but fuzzy evidence, while a DS mass function is a model of uncertain and crisp evidence.  In general, a pair $(X,\Pi_{(Y\mid X)})$ defines an \emph{epistemic random fuzzy set}, allowing us to describe evidence that is both uncertain and fuzzy. (The term ``epistemic'' emphasizes the distinction between this interpretation and that of random fuzzy sets as mechanisms for generating fuzzy data  considered, for instance in \cite{puri86,gil06}). 
In  \cite{denoeux21a}, we have proposed a family of combination rules for epistemic random fuzzy sets in the finite setting, generalizing both Dempster's rule and the conjunctive combination rules of possibility theory. One of these rules, based on the product t-norm, is associative and  arguably well suited for combining independent evidence. Equipped with this combination rule (called here the \emph{generalized product-intersection rule}), the theory of epistemic random fuzzy sets can be seen as an extension of both DS theory and possibility theory, making it possible to combine evidence of various types, including expert assessments (possibly expressed in natural language), sensor information, and statistical evidence about a model parameter.

In this paper, drawing from mathematical results presented by Couso and S\'anchez in \cite{couso11}, we give a more general exposition of  the theory of epistemic fuzzy sets, considering arbitrary probability and measurable spaces. We define  combination, marginalization and vacuous extension operations of random fuzzy sets in this general setting, laying the foundations of a wide-ranging theory of uncertainty encompassing DS and possibility theories as special cases. Finally, for the important case where the frame of discernment is $\reels^p$, we propose  Gaussian random fuzzy numbers and vectors as a practical model, generalizing both Gaussian random variables and vectors on the one hand, and Gaussian possibility distributions on the other hand.

The rest of this paper is organized as follows. Classical models (including random sets, fuzzy sets and possibility theory) are first recalled in Section \ref{sec:classical}. Epistemic random fuzzy sets are then introduced in a general setting in Section \ref{sec:random_fuzzy}. Finally, Gaussian random fuzzy numbers and vectors are studied, respectively, in Sections \ref{sec:GRFS} and \ref{sec:GRFV}, and Section \ref{sec:concl} concludes the paper.

\section{Classical models}
\label{sec:classical}

In this section, we recall the main definitions and results pertaining to the  two  models of uncertainty generalized in this paper:  random sets and belief functions on the one hand (Section \ref{subsec:random_sets}), fuzzy sets and possibility theory on the other hand (Section \ref{subsec:fuzzy}).

\subsection{Random sets and belief functions}
\label{subsec:random_sets}

Whereas belief functions in the finite setting can be introduced without any reference to random sets \cite{shafer76}, the mathematical framework of random sets is  useful to analyze belief functions in more general spaces, and to define the practical models needed, e.g., in  statistical applications. Important references about the link between random sets and belief functions include \cite{nguyen78} and \cite{couso11}.


Let $(\Omega,\sigma_\Omega,P)$ be a probability space,  $(\Theta,\sigma_\Theta)$ a measurable space, and $\bbX$ a  mapping from $\Omega$ to $2^\Theta$. The \emph{upper} \emph{and lower inverses} of $\bbX$ are defined, respectively, as follows:
\begin{subequations}
\begin{align}
\bbX^*(B)&=B^*=\{\omega \in \Omega : \bbX(\omega)\cap B\neq\emptyset\}\\
\bbX_*(B)&=B_*=\{\omega \in \Omega : \emptyset\neq\bbX(\omega)\subseteq B\}
\end{align}
\end{subequations}
for all $B\subseteq \Theta$. It is easy to check that 
\[
B^*\cap (B^c)_*=\emptyset
\]
and
\[
B^*\cup (B^c)_*=\{\omega\in \Omega: \bbX(\omega)\neq \emptyset\}=\Theta^*,
\]
where $B^c$ denotes the complement of $B$ in $\Theta$.

The mapping $\bbX$ is said to be $\sigma_\Omega-\sigma_\Theta$ \emph{strongly measurable} \cite{nguyen78} if, for all $B\in \sigma_\Theta$,  $B^*\in \sigma_\Omega$ (or, equivalently, if for all $B\in \sigma_\Theta$,  $B_*\in \sigma_\Omega$). The tuple $(\Omega,\sigma_\Omega,P,\Theta,\sigma_\Theta,\bbX)$ is called a \emph{random set}. When there is no confusion about the domain and co-domain, we will call the $\sigma_\Omega-\sigma_\Theta$ strongly measurable mapping $\bbX$ itself a \emph{random set}. 

In the special case where  $|\bbX(\omega)|=1$ for all $\omega\in\Omega$, we can define the  mapping $X:\Omega\rightarrow \Theta$ such that $\bbX(\omega)=\{X(\omega)\}$ for all $\omega\in\Omega$. We then have $B^*=B_*=X^{-1}(B)$ for all $B\subseteq\Theta$, and $X$ is $\sigma_\Omega-\sigma_\Theta$ measurable. The notion of random set thus extends that of random variable.


\paragraph{Belief and plausibility functions} From now on, we will assume, for simplicity, that $P(\Theta^*)=1$. (If not verified, this property can be enforced by conditioning $P$ on $\Theta^*$). Let $P^*$ and $P_*$ be the lower and upper probability measures associated with random set $\bbX$, defined as the mappings from $\sigma_\Theta$ to $[0,1]$ such that
\begin{equation}
P_*(B)=P(B_*) 
\end{equation}
and
\begin{equation}
P^*(B)=P(B^*)=1-P_*(B^c),
\end{equation}
for all $B\in \sigma_\Theta$. Mapping $P_*$ is a completely monotone capacity, i.e., a \emph{belief function}, and $P^*$  is the dual \emph{plausibility function} \cite[Proposition 1]{nguyen78}. In the following, they will be denoted, respectively, as $Bel_\bbX$ and $Pl_\bbX$. The corresponding \emph{contour function} is defined as the mapping $pl_\bbX$ from $\Theta$ to $[0,1]$ such that
\[
pl_\bbX(\theta)=Pl_\bbX(\{\theta\})
\]
for all $\theta\in\Theta$. The subsets $\bbX(\omega)\subseteq \Theta$, for all $\omega\in \Omega$, are called the \emph{focal sets} of $Bel_\bbX$.

\paragraph{Interpretation} In DS theory, $\Omega$ represents a set of interpretations of a piece of evidence about a variable $\btheta$ taking values in set $\Theta$ (called the \emph{frame of discernment}). If interpretation $\omega\in\Omega$ holds, we know that $\btheta\in \bbX(\omega)$, and nothing more.  For any $A\in \sigma_\Omega$, $P(A)$ is the (subjective) probability that the true interpretation lies in $A$. For any $B\in\sigma_\Theta$, the degree of belief $Bel_\bbX(B)$ is then a measure of support of the proposition ``$\btheta \in B$'' given the evidence, while the degree of plausibility $Pl_\bbX(B)$ is a measure of lack of support for the proposition ``$\btheta \not\in B$''. Under this interpretation, the random set $\bbX$ represents a state of knowledge: it can be said to be \emph{epistemic}. 

\paragraph{Vacuous random set} A constant random set $(\Omega,\sigma_\Omega,P,\Theta, \sigma_\Theta,\bbX)$ such that $\bbX(\omega)=\Theta$ for all $\omega\in\Omega$ is said to be \emph{vacuous}. For such a random set, we have $Bel_\bbX(A)=0$ for all $A\in \sigma_\Theta\setminus \{\Theta\}$ and  $Pl_\bbX(A)=1$ for all $A\in \sigma_\Theta\setminus \{\emptyset\}$. A vacuous random set represents complete ignorance about $\btheta$.

\paragraph{Finite case} 
Assume that $\Theta$ is finite, and $\sigma_\Theta=2^\Theta$. The \emph{M\"obius inverse} of $Bel_\bbX$ is the mapping $m_\bbX$ from $2^\Theta$ to [0,1] such that 
\[
m_\bbX(B)=\sum_{C \subseteq B} (-1)^{|B|-|C|} Bel_\bbX(C), 
\]
for all $B\subseteq \Theta$. It verifies $m(B)\ge 0$ for all $B\subseteq \Omega$,  $\sum_{B\subseteq\Omega} m(B)=1$ and $m(\emptyset)=0$. The belief and plausibility can be computed from $m_\bbX$, respectively, as
\[
Bel_\bbX(B)=\sum_{C\subseteq B} m_\bbX(C) \quad  \text{and} \quad 
Pl_\bbX(B)=\sum_{C\cap B\neq\emptyset} m_\bbX(C), 
\]
for all $B\subseteq\Theta$.

\paragraph{Random closed intervals} 
Random closed intervals are particularly simple models allowing us to define belief functions on the real line \cite{dempster68b,smets05a,denoeux09b}.  Let $(\Omega,\sigma_\Omega,P)$ be a probability space and $X,Y$ two random variables $\Omega\rightarrow \reels$ such that $P(\{\omega\in \Omega: X(\omega)\le Y(\omega)\})=1$. Then, the mapping $\bbX:\Omega\rightarrow 2^\reels$ defined by $\bbX(\omega)=[X(\omega),Y(\omega)]$ is $\sigma_\Omega-\beta_\reels$ strongly measurable, where  $\beta_\reels$ is the Borel $\sigma$-algebra on $\reels$ (see a formal proof in \cite{miranda05}). 
This mapping defines a \emph{random closed interval}. For a random closed interval $\bbX=[X,Y]$, we have \cite{dempster68b}
\begin{subequations}
\label{eq:BelPlxy}
\begin{equation}
Bel_\bbX([x,y])=P([X,Y]\subseteq [x,y])= P(X\ge x; Y\le y) 
\end{equation}
and
\begin{equation}
Pl_\bbX([x,y])=P([X,Y]\cap [x,y]\neq\emptyset)= 1-P(X> y)-P(Y< x),
\end{equation}
\end{subequations}
for all $(x,y)\in\reels^2$ such that $x\le y$. \new{In particular, by letting $x$ tend to $-\infty$ in \eqref{eq:BelPlxy}, we obtain the \emph{lower and upper cumulative distribution functions (cdf's)} of $\bbX$ as
\begin{subequations}
\label{eq:exprandomint}
\begin{equation}
F_*(y)=Bel_\bbX((-\infty,y])=P(Y\le y)=F_Y(y)
\end{equation}
and
\begin{equation}
F^*(y)=Pl_\bbX((-\infty,y])=P(X\le y)=F_X(y).
\end{equation}
\end{subequations}
}
\new{
\paragraph{Lower and upper expectation} Let $\bbX$ be a random set from $(\Omega,\sigma_\Omega,P)$ to $(\reels,\beta_\reels)$. Following Dempster \cite{dempster67a}, we can define its  \emph{lower and upper expectations}, respectively, as
\[
\esp_*(\bbX)=\int_{-\infty}^{+\infty} x \, dF^*(x) 
\]
and
\[
\esp^*(\bbX)=\int_{-\infty}^{+\infty} x \, dF_*(x),  
\]
where $F_*(x)=Bel_\bbX((-\infty,x])$ and $F^*(x)=Pl_\bbX((-\infty,x])$ are the lower and upper cdf's of $\bbX$. When $\bbX$ is a random closed interval $[X,Y]$, it follows from \eqref{eq:exprandomint} that $\esp_*(\bbX)=\esp(X)$ and  $\esp^*(\bbX)=\esp(Y)$.
}

\paragraph{Dempster's rule} 

Consider two pieces of evidence represented by  random sets 
\[
(\Omega_1, \sigma_1,P_1, \Theta, \sigma_\Theta, \bbX_1) \quad \text{and} \quad (\Omega_2,\sigma_2,P_2,\Theta, \sigma_\Theta,\bbX_2),
\] 
and the mapping $\bbX_\cap$ from $\Omega_1\times\Omega_2$ to $2^\Theta$ defined by $\bbX_\cap(\omega_1,\omega_2)=\bbX_1(\omega_1)\cap \bbX_2(\omega_2)$.
If interpretations $\omega_1\in \Omega_1$ and $\omega_2\in \Omega_2$ both hold, we know that $\btheta\in \bbX_\cap(\omega_1,\omega_2)$, provided that $\bbX_1(\omega_1)\cap \bbX_2(\omega_2)\neq\emptyset$.  Assume that $\bbX_\cap$ is  $(\sigma_1\otimes\sigma_2)-\sigma_\Theta$ strongly measurable, where $\sigma_1\otimes\sigma_2$ is the tensor product $\sigma$-algebra over the Cartesian product $\Omega_1\times\Omega_2$. The two pieces of evidence are said to be \emph{independent} if, for any $A\in \sigma_1\otimes \sigma_2$, the probability that  $A$ contains the  true interpretations of the two pieces of evidence is the conditional probability
\begin{equation}
\label{eq:P12}
P_{12}(A)=(P_1\times P_2)(A \mid \Theta^*)=\frac{(P_1\times P_2)(A\cap\Theta^*)}{(P_1\times P_2)(\Theta^*)},
\end{equation}
where $P_1\times P_2$ is the product measure satisfying  $(P_1\times P_2)(A_1\times A_2)=P_1(A_1)P_2(A_2)$ for all $A_1\in \sigma_1$, $A_2\in \sigma_2$, and
\[
\Theta^*=\{(\omega_1,\omega_2)\in \Omega_1\times \Omega_2 : \bbX_\cap(\omega_1,\omega_2)\neq \emptyset\}
\]
is the set of noncontradictory pairs of interpretations. The quantity
\[
\kappa=1-(P_1\times P_2)(\Theta^*)=(P_1\times P_2)(\{(\omega_1,\omega_2)\in \Omega_1\times \Omega_2 : \bbX_\cap(\omega_1,\omega_2)= \emptyset\})
\]
is called the \emph{degree of conflict} between the two pieces of evidence. The combined  random set 
\[
(\Omega_1\times\Omega_2,\sigma_1\otimes\sigma_2,P_{12},\Theta, \sigma_\Theta,\bbX_\cap)
\]
is called the \emph{orthogonal sum} of the two pieces of evidence, and is denoted by $\bbX_1\oplus\bbX_2$. This combination rule, first introduced by Dempster in \cite{dempster67a}, is called the \emph{product-intersection rule}, or \emph{Dempster's rule of combination}. 

We can remark that Dempster's rule is usually viewed as an operation to combine belief functions, whereas it is defined here as an operation to combine \emph{random sets}. This distinction is immaterial in the standard setting, as the orthogonal sum of two belief functions does not depend on their particular random set representations and can be defined without reference to the random set framework \cite{shafer16b}. However, it becomes crucial when considering  random fuzzy sets as a model for generating belief functions, as done in this paper. We will come back to this important point in Section \ref{subsec:genDS}.

Any vacuous random set is obviously a neutral element for Dempster's rule. The following important proposition states that pieces of evidence can be combined by Dempster's rule in any order.

\begin{Prop}
\label{prop:dempster_assoc}
Dempster's rule is commutative and associative.
\end{Prop}
\begin{proof}
See \ref{app:dempster_assoc}.
\end{proof}

\begin{Ex}
\label{ex:dempster}
Let $X_1 \sim N(\mu_1,\sigma_1^2)$ and $X_2 \sim N(\mu_2,\sigma_2^2)$ be two independent normal random variables and consider the random intervals $\bbX_1=[X_1,+\infty)$ and $\bbX_2=(-\infty,X_2]$. The degree of conflict between $\bbX_1$ and $\bbX_2$ is
\[
\kappa=P(X_1>X_2)= P(X_2-X_1<0)=\Phi\left(\frac{\mu_1-\mu_2}{\sqrt{\sigma_1^2+\sigma_2^2}}\right),
\]
where $\Phi$ is the standard normal cdf. The orthogonal sum of $\bbX_1$ and $\bbX_2$ is the random closed interval $[X'_1,X'_2]$, where $(X'_1,X'_2)$ is the two-dimensional random vector  with distribution equal the conditional distribution of $(X_1,X_2)$ given $X_1\le X_2$. Its density is
\[
f_{X'_1,X'_2}(x_1,x_2)=\frac{\sigma_1^{-1}\sigma_2^{-1}\phi\left(\frac{x_1-\mu_1}{\sigma_1}\right)\phi\left(\frac{x_2-\mu_2}{\sigma_2}\right) I(x_1\le x_2)}{\Phi\left(\frac{\mu_2-\mu_1}{\sqrt{\sigma_1^2+\sigma_2^2}}\right)},
\]
where $\phi$ is the standard normal probability density function (pdf) and $I(\cdot)$ is the indicator function.
\end{Ex}

The following proposition states that the contour function of the orthogonal sum of two independent random sets $\bbX_1$ and $\bbX_2$ is proportional to the product of the contour functions of $\bbX_1$ and $\bbX_2$.
%
%
\begin{Prop}
\label{prop:prodQ}
Let $\bbX_1$ and $\bbX_2$ be two independent random sets on the same domain $\Theta$, with contour functions $pl_{\bbX_1}$ and  $pl_{\bbX_2}$. 
For any $\theta\in\Theta$, 
\begin{equation}
\label{eq:prodpl}
pl_{\bbX_1\oplus \bbX_2}(\theta) =\frac{pl_{\bbX_1}(\theta) pl_{\bbX_2}(\theta)}{1-\kappa},
\end{equation}
where $\kappa$ is the degree of conflict between $\bbX_1$ and $\bbX_2$.
\end{Prop}
\begin{proof} We have
\begin{align*}
pl_{\bbX_1\oplus \bbX_2}(\theta)&=\frac{(P_1\times P_2)(\{(\omega_1,\omega_2)\in \Omega_1\times\Omega_2: \theta\in \bbX_\cap(\omega_1,\omega_2)\})}{1-\kappa}\\
&= \frac{(P_1\times P_2)(\{\omega_1\in \Omega_1: \theta\in \bbX_1(\omega_1)\} \times \{\omega_2\in \Omega_2: \theta\in \bbX_2(\omega_2)\} )}{1-\kappa}\\
&=\frac{P_1(\{\omega_1\in \Omega_1: \theta\in \bbX_1(\omega_1)\}) \cdot P_2(\{\omega_2\in \Omega_2: \theta\in \bbX_2(\omega_2)\})}{1-\kappa}\\
& = \frac{pl_{\bbX_1}(\theta) pl_{\bbX_2}(\theta)}{1-\kappa}.
\end{align*}
\end{proof}

\begin{Ex}
Let us consider again the two random intervals of Example \ref{ex:dempster}. The contour functions of $\bbX_1$ and $\bbX_2$ are, respectively,
\[
pl_{\bbX_1}(x)=P(X_1\le x)=\Phi\left(\frac{x-\mu_1}{\sigma_1}\right)
\]
and
\[
pl_{\bbX_2}(x)=P(X_2\ge x)=1-\Phi\left(\frac{x-\mu_2}{\sigma_2}\right).
\]
Now, the contour function of $\bbX_1\oplus\bbX_2$ is
\begin{align*}
pl_{\bbX_1\oplus \bbX_2}(x)&=P(X'_1\le x \le X'_2)\\
&=\int_{-\infty}^x\int_x^{+\infty} f_{X'_1,X'_2}(x_1,x_2) dx_2dx_1\\
&=\left[\Phi\left(\frac{\mu_2-\mu_1}{\sqrt{\sigma_1^2+\sigma_2^2}}\right)\right]^{-1} \int_{-\infty}^x\int_x^{+\infty}\sigma_1^{-1}\sigma_2^{-1}\phi\left(\frac{x_1-\mu_1}{\sigma_1}\right)\phi\left(\frac{x_2-\mu_2}{\sigma_2}\right)dx_2dx_1\\
&=\frac{\Phi\left(\frac{x-\mu_1}{\sigma_1}\right)\left[1-\Phi\left(\frac{x-\mu_2}{\sigma_2}\right)\right]}{\Phi\left(\frac{\mu_2-\mu_1}{\sqrt{\sigma_1^2+\sigma_2^2}}\right)}=\frac{pl_{\bbX_1}(x) pl_{\bbX_2}(x)}{1-\kappa}.
\end{align*}
\end{Ex}

\paragraph{Marginalization and vacuous extension} Let us now consider the case where we have two variables $\btheta_1$ and $\btheta_2$ with domains  $\Theta_1$ and $\Theta_2$. (The case of $n$ variables is not more difficult conceptually but it requires heavier notations). Let $\sigma_{\Theta_1}$ and $\sigma_{\Theta_2}$ be $\sigma$-algebras defined, respectively,  on $\Theta_1$ and $\Theta_2$. Let $\Theta_{12}=\Theta_1\times\Theta_2$ and $\sigma_{\Theta_{12}}=\sigma_{\Theta_1}\otimes\sigma_{\Theta_2}$.  Let $\bbX_{12}$ be a random set from $(\Omega,\sigma_\Omega,P)$ to $(\Theta_{12},\sigma_{\Theta_{12}})$, and $\bbX_1$ the mapping from $\Omega$ to $2^{\Theta_1}$ that maps each $\omega\in\Omega$ to the \emph{projection} of $\bbX_{12}(\omega)$ onto $\Theta_1$:
\[
\bbX_1(\omega)=\bbX_{12}(\omega)\downarrow \Theta_1=\{\theta_1\in \Theta_{1}: \exists \theta_2\in\Theta_2, (\theta_1,\theta_2)\in \bbX_{12}(\omega)\}.
\]
It is easy to see that $\bbX_1$ is $\sigma_\Omega-\sigma_{\Theta_1}$ measurable: for any $B\in \sigma_{\Theta_1}$,
\begin{align*}
\bbX_1^*(B)&=\{\omega\in\Omega: \bbX_1(\omega)\cap B\neq\emptyset\}\\
&=\{\omega\in\Omega: \bbX_{12}(\omega)\cap (B\times\Theta_2)\neq\emptyset\}\\
&=\bbX_{12}^*(B\times\Theta_2).
\end{align*}
As $B\times\Theta_2\in \sigma_{\Theta_{12}}$ and $\bbX_{12}$ is $\sigma_\Omega-\sigma_{\Theta_{12}}$ strongly measurable, it thus follows that $\bbX_1^*(B)\in\sigma_\Omega$. The random set $\bbX_1$ will be called the \emph{marginal} of $\bbX_{12}$ on $\Theta_1$. 

Conversely, let $\bbX_1$ be a random set from $(\Omega,\sigma_\Omega)$ to $(\Theta_1,\sigma_{\Theta_1})$ and let $\bbX_{1\uparrow 2}$ be the mapping from $\Omega$ to $\Theta_{12}$ defined by
\[
\bbX_{1\uparrow (1,2)}(\omega)=\bbX_1(\omega)\times\Theta_2.
\]
For any $B\in\sigma_{\Theta_{12}}$,
\begin{align*}
\bbX_{1\uparrow (1,2)}^*(B)&=\{\omega\in\Omega: \bbX_{1\uparrow 2}(\omega)\cap B\neq\emptyset\}\\
&=\{\omega\in\Omega: \bbX_{1}(\omega)\cap (B\downarrow \Theta_1)\neq\emptyset\}\\
&=\bbX_1^*(B\downarrow \Theta_1).
\end{align*}
If for all $B\in\sigma_{\Theta_{12}}$, $\bbX_1^*(B\downarrow \Theta_1)\in\sigma_\Omega$, then $\bbX_{1\uparrow (1,2)}$ is  $\sigma_\Omega-\sigma_{\Theta_{12}}$ strongly measurable. It is said to be the \emph{vacuous extension} of $\bbX_1$ in $\Theta_1\times\Theta_2$.

We say that a random set $\bbX_{12}$ from $(\Omega,\sigma_\Omega,P)$ to $(\Theta_{12},\sigma_{\Theta_{12}})$ with marginals $\bbX_1$ and $\bbX_2$ is \emph{noninteractive} if it is equal to the orthogonal sum of its marginals, i.e.,
\[
\bbX _{12}= \bbX_{1\uparrow (1,2)} \oplus \bbX_{2\uparrow (1,2)} \quad \text{denoted by} \quad \bbX_{1} \oplus \bbX_{2} .
\]

\begin{Ex} Let $(X_1,X_2)$ be a two dimensional random vector from $(\Omega,\sigma_\Omega,P)$ to $(\reels^2,\beta_{\reels^2})$ and consider the mapping $\bbX_{12}: \Omega \rightarrow 2^{\reels^2}$ defined as
\[
\bbX_{12}(\omega)=(-\infty,X_1(\omega)] \times (-\infty,X_2(\omega)]. 
\]
This mapping defines a random set \cite[page 3]{molchanov05}. Its marginals are the random closed intervals $(-\infty,X_1]$ and $(-\infty,X_2]$. If $X_1$ and $X_2$ are independent, then $\bbX_{12}=(-\infty,X_1]\oplus (-\infty,X_2]$ and $\bbX_{12}$ is noninteractive.
\end{Ex}

\subsection{Fuzzy sets and possibility theory}
\label{subsec:fuzzy}

A \emph{fuzzy subset} of a  set $\Theta$ is a pair $\tF=(\Theta,\mu_{\tF})$, where $\mu_{\tF}$ is a mapping from $\Theta$ to $[0,1]$, called the \emph{membership function} of $\tF$. Each number  $\mu_{\tF}(\theta)$ is interpreted as a degree of  membership  of element $\theta$ to the fuzzy set $\tF$. In the following, to simplify the notation, we will identify fuzzy sets to their membership functions and write $\tF(\theta)$ for $\mu_{\tF}(\theta)$. The height of fuzzy set $\tF$ is defined as
\[
\height(\tF)=\sup_{\theta\in\Theta} \tF(\theta).
\]
If $\height(\tF)=1$, $\tF$ is said to be \emph{normal}. For any $\alpha\in[0,1]$, the (weak) $\alpha$-cut of $\tF$ is the set
\[
\cut{\alpha}{\tF}=\{\theta \in \Theta: \tF(\theta)\ge \alpha\}.
\]

\paragraph{Possibility and necessity measures} Let $\btheta$ be a variable taking values in $\Theta$. Assume that we receive a piece of evidence telling us that ``$\btheta$ is $\tF$'', where $\tF$ is a normal fuzzy subset of $\Theta$. This evidence induces a \emph{possibility measure} $\Pi_\tF$ from $2^\Theta$ to $[0,1]$ defined by
\begin{equation}
\label{eq:Pi}
\Pi_\tF(B)=\sup_{\theta\in B} \tF(\theta),
\end{equation}
for all $B\subseteq \Theta$. The number $\Pi_\tF(B)$ is interpreted as the degree of possibility that $\btheta\in B$, given that $\btheta$ is $\tF$ \cite{zadeh78}. The corresponding \emph{possibility distribution} is the mapping from $\Theta$ to $[0,1]$ defined by
\[
\pi_\tF(\theta)=\Pi_\tF(\{\theta\})=\tF(\theta),
\] 
i.e., it is identical to the membership function $\tF$. The dual \emph{necessity measure} is defined as
\begin{equation}
\label{eq:N}
N_\tF(B)=1-\Pi_\tF(B^c)= \inf_{\theta\not\in B} \left[1-\tF(\theta)\right].
\end{equation}
It can easily be shown  that  mapping $N_\tF: 2^\Omega \rightarrow [0,1]$  is completely monotone,  i.e., it is a belief function, and $\Pi_\tF$ is the dual plausibility function \cite{dubois98b}. These belief and plausibility functions are formally induced by the random set $([0,1],\beta_{[0,1]},\lambda,\Theta,2^\Theta,\bbX)$, where $\beta_{[0,1]}$ is the Borel $\sigma$-algebra on $[0,1]$, $\lambda$ is the uniform probability measure, and $\bbX$ is the mapping $[0,1]\rightarrow 2^\Theta$ defined by $\bbX(\alpha)=\cut{\alpha}{\tF}$. However, as we will see in Section \ref{subsec:genDS}, it is important, when combining evidence, to distinguish between possibility distributions induced by fuzzy sets, and consonant belief functions induced by random sets.


\paragraph{Conjunctive combination of possibility distributions} Assume that we receive two independent pieces of information telling us that ``$\btheta$ is $\tF$'' and ``$\btheta$ is $\tG$'', where $\tF$ and $\tG$ are two fuzzy subsets of $\Theta$. The conjunctive combination of these two pieces of evidence requires some notion of intersection between fuzzy sets.  As reviewed in \cite{dubois00a}, the intersection operation can be extended to fuzzy sets using triangular norms (or t-norms for short). Given a t-norm $\top$, the $\top$-intersection of two fuzzy subsets $\tF$ and $\tG$ of the same domain $\Theta$ can be defined as
\[
(\tF \cap_\top \tG)(\theta)=\tF(\theta)\top \tG(\theta)
\]
for all $\theta\in\Theta$. The most common choices for $\top$ are the minimum and product t-norms, as originally proposed by Zadeh \cite{zadeh65}; the corresponding operations are called, respectively, the \emph{minimum} and \emph{product} intersections. However, the intersection of two normal fuzzy sets is generally not normal. To obtain a normal fuzzy set, as needed for the definitions of possibility and necessity measures in \eqref{eq:Pi}-\eqref{eq:N}, we define the normalized $\top$-intersection as
\[
(\tF \cap^*_\top \tG)(\theta)=\begin{cases}
\displaystyle \frac{\tF(\theta)\top \tG(\theta)}{\height(\tF \cap_\top \tG)} & \text{if } \height(\tF \cap_\top \tG)>0\\
0 & \text{otherwise.} 
\end{cases}
\]
The fuzzy set $\tF \cap^*_\top \tG$ is normal provided that $\height(\tF \cap_\top \tG)>0$. In general, the normalized intersection $\cap^*_\top$ associated with a t-norm $\top$ is not associative. A notable exception is the case where $\top$ is the product t-norm:  the normalized product intersection, denoted by $\varodot$, is associative (see \cite{dubois99}, and a simple proof in \cite{denoeux21a}). By  abuse of notation, we can use the same symbol to denote the conjunctive combination of possibility measures and the normalized product intersection of fuzzy sets, and  write
\[
\Pi_\tF \varodot \Pi_\tG=\Pi_{\tF\varodot\tG}.
\]
As noted by Dubois and Prade \cite[page 352]{dubois99}, product intersection has a reinforcement effect that is appropriate when the information sources are assumed to be independent. The choice of the normalized product intersection for combining possibility distributions makes possibility theory fit in the framework of valuation-based systems \cite{shenoy92} and allows for possibilistic reasoning with a large number of variables. The normalized product intersection operator also has an interesting property with respect to Gaussian fuzzy numbers, as recalled in the next paragraph.

\paragraph{Gaussian fuzzy numbers} A \emph{fuzzy number} (or fuzzy interval) can be defined as a normal and convex fuzzy subset of the real line. In particular, a \emph{Gaussian fuzzy number (GFN)}  is a normal fuzzy subset of $\reels$ with membership function
\[
\varphi(x;m,h)=\exp\left(-\frac{h}{2}(x-m)^2\right),
\]
where  $m\in\reels$ is the mode and  $h\in [0,+\infty]$ is the precision. Such a fuzzy number will be denoted by $\GFN(m,h)$. If $h=0$, $\varphi(x;m,h)=1$ for all $x\in \reels$: $\GFN(m,0)$ is then \new{maximally imprecise and} identical to the whole real line, whatever the value of $m$.  \new{If $h=+\infty$, $\varphi(x;m,h)=I(x=m)$, where $I(\cdot)$ is the indicator function; the fuzzy number $\GFN(m,+\infty)$  is then maximally precise and  equivalent to the real number $m$.}

It can easily be shown that the family of GFN's is closed under the normalized product intersection (see, e.g., \cite{bromiley14}). More precisely, we have the following proposition, proved in  \cite{bromiley14}.

\begin{Prop}
\label{prop:prodphi}
For any $x \in \reels$,
\[
\varphi(x;m_1,h_1) \cdot \varphi(x;m_2,h_2)=  \exp\left(-\frac{h_1h_2(m_1-m_2)^2}{2(h_1+h_2)}\right) \varphi(x;m_{12},h_{12}) ,
\]
with 
\[
m_{12}=\frac{h_1 m_1+h_2 m_2}{h_1+h_2} \quad \text{and} \quad h_{12}=h_1+h_2.
\]
Consequently, 
\[
\GFN(m_1,h_1)  \varodot \GFN(m_2,h_2)= \GFN(m_{12},h_{12}),
\]
and 
\begin{equation}
\label{eq:eta1D}
\height\left[\GFN(m_1,h_1) \cdot \GFN(m_2,h_2)\right]= \exp\left(-\frac{h_1h_2(m_1-m_2)^2}{2(h_1+h_2)}\right).
\end{equation}
\end{Prop}

\paragraph{Marginalization and cylindrical extension} Let us now assume that we have two variables $\btheta_1$ and $\btheta_2$ jointly constrained by a possibility distribution $\pi_\tF$, where $\tF$ is a fuzzy subset of $\Theta_{12}=\Theta_1\times\Theta_2$. As a result of \eqref{eq:Pi}, variable $\btheta_1$ alone is constrained by the possibility distribution
\[
\pi_1(\theta_1)=\Pi(\{\theta_1\}\times \Theta_2)=\sup_{\theta_2\in\Theta_2} \pi_\tF(\theta_1,\theta_2)= \sup_{\theta_2\in\Theta_2} \tF(\theta_1,\theta_2)=(\tF\downarrow\Theta_1)(\theta_1),
\]
where $\tF\downarrow\Theta_1$ is the \emph{projection} of $\tF$ on $\Theta_1$. We say that $\pi_1$ is the \emph{marginal} of $\pi_\tF$ on $\Theta_1$. Conversely, given a possibility distribution $\pi_{\tF_1}$, where $\tF_1$ is a fuzzy subset  of $\Theta_1$, its \emph{cylindrical extension} in $\Theta_1\times\Theta_2$ is the possibility distribution $\pi_{\tF_1\times\Theta_2}$ defined as
\[
\pi_{\tF_1\times\Theta_2}(\theta_1,\theta_2)=\pi_{\tF_1}(\theta_1)
\]
for all $(\theta_1,\theta_2)\in\Theta_1\times\Theta_2$. We say that the joint possibility distribution $\pi_{\tF}$ on $\Theta_{12}$ is \emph{noninteractive} with respect to the product intersection if it is the product of its marginals:
\[
\pi_{\tF}(\theta_1,\theta_2)=\pi_{\tF\downarrow\Theta_1}(\theta_1) \cdot \pi_{\tF\downarrow\Theta_2}(\theta_2).
\]

\begin{Ex}
Let $\pi_{12}$ be the possibility distribution on $\reels^2$ defined by
\begin{align*}
\pi_{12}(x_1,x_2)&=\exp\left(-\frac{h_1}2(x_1-m_1)^2-\frac{h_2}2(x_2-m_2)^2 \right)\\
&=\exp\left(-\frac{h_1}2(x_1-m_1)^2\right)\exp\left(-\frac{h_2}2(x_2-m_2)^2 \right).
\end{align*}
Its marginals are
\[
\pi_1(x_1)=\max_{\theta_2}\pi_{12}(x_1,x_2)=\exp\left(-\frac{h_1}2(x_1-m_1)^2\right)
\]
and
\[
\pi_2(x_2)=\max_{\theta_1}\pi_{12}(x_1,x_2)=\exp\left(-\frac{h_2}2(x_2-m_2)^2\right).
\]
Consequenty, $\pi_{12}$ is noninteractive with respect to the product intersection.
\end{Ex}

\section{Epistemic random fuzzy sets}
\label{sec:random_fuzzy}

The proposed epistemic random fuzzy set model is introduced in this section. The main definitions are first given in Section \ref{subsec:def}, and the generalized product-intersection rule is introduced in Section \ref{subsec:genDS}. Marginalization and vacuous extension are then addressed in Section \ref{subsec:margin}, and an application to statistical inference is briefly discussed in Section \ref{subsec:stat}.

\subsection{General definitions}
\label{subsec:def}


As before, let $(\Omega,\sigma_\Omega,P)$ be a probability space and let $(\Theta,\sigma_\Theta)$ be a measurable space. Let $\tX$ by a mapping from $\Omega$ to the set $[0,1]^\Theta$ of fuzzy subsets of $\Theta$. For any $\alpha\in [0,1]$, let $\cut{\alpha}{\tX}$ be the mapping from $\Omega$ to $2^\Theta$ defined as
\[
\cut{\alpha}{\tX}(\omega)=\cut{\alpha}{[\tX(\omega)]},
\]
where $\cut{\alpha}{[\tX(\omega)]}$  is the weak $\alpha$-cut of $\tX(\omega)$.  If for any $\alpha\in [0,1]$, $\cut{\alpha}{\tX}$ is $\sigma_\Omega-\sigma_\Theta$ strongly measurable,  the tuple  $(\Omega,\sigma_\Omega,P,\Theta, \sigma_\Theta,\tX)$  is said to be a \emph{random fuzzy set} (also called a \emph{fuzzy random variable})  \cite{couso11}. It is clear that the class of random fuzzy sets includes that of random sets, just as the class of fuzzy sets includes that of classical (crisp) sets.   


\begin{Ex}
\label{ex:RFS}
Let $M$ be  a Gaussian random variable from $(\Omega,\sigma_\Omega,P)$ to $(\reels,\beta_\reels)$, with mean $\mu$ and standard deviation $\sigma$, and let $\tX$ be the mapping from $\Omega$ to $[0,1]^\reels$ that maps each $\omega\in \Omega$ to the triangular fuzzy number with mode $M(\omega)$ and support $[M(\omega)-a,M(\omega)+a]$:
\[
\tX(\omega)(x)=\begin{cases}
\frac{a- |x-M(\omega)|}{a} & \text{if } |x-M(\omega)|\le a\\
0 & \text{otherwise. }
\end{cases}
\]
for some $a>0$. For any $\alpha\in[0,1]$, the $\alpha$-cut of $\tX(\omega)$ is
\[
\cut{\alpha}{\tX}(\omega)=\left[M(\omega)-a(1-\alpha),M(\omega)+a(1-\alpha)\right].
\]
The random set $\cut{\alpha}{\tX}: \omega\rightarrow \cut{\alpha}{\tX}(\omega)$ is $\sigma_\Omega-\beta_\reels$ strongly measurable (it is a random closed interval). Consequently, $\tX$ is a random fuzzy set. \new{In the following, such random fuzzy sets with domain $[0,1]^\reels$ will be called \emph{random fuzzy numbers}.}
\end{Ex}

\paragraph{Interpretation} Here, as in \cite{denoeux21a}, we use random fuzzy sets as a model of unreliable and fuzzy evidence. In this model, we see $\Omega$ as a \emph{set of interpretations} of a piece of evidence about a variable $\btheta$ taking values in $\Theta$. If interpretation $\omega\in \Omega$ holds, we know that ``$\btheta \text{ is } \tX(\omega)$'', i.e., $\btheta$  is constrained by the possibility distribution $\pi_{\tX(\omega)}$. We qualify such random fuzzy sets as \emph{epistemic}, because they encode a state of knowledge about some variable $\btheta$. \new{It should be noted that this semantics of random fuzzy sets is different from those reviewed in \cite{couso11}. The conditional possibility  interpretation developed in \cite{couso11} is the closest to ours, since we also see the fuzzy sets $\tX(\omega)$ as defining conditional possibility measures. However, in \cite{couso11}, the authors use the random fuzzy set formalism to model a situation in which we have two random experiments, one of which is completely determined; the family of possibility distributions $\{\pi_{\tX(\omega)}: \omega\in\Omega\}$  then models our knowledge about the relationship between the outcomes $\omega$ of the first experiment and the possible outcomes of the second one. This formalism allows the authors of \cite{couso11} to compute lower and upper bounds on the probability of any event related to the second experiment. In contrast, our model does not rely on the notion of random experiment. In particular, we do not postulate the existence of an objective probability measure on $\Theta$, and the belief and plausibility measures introduced below are not interpreted as lower and upper bounds on ``true'' probabilities.}  

\paragraph{Belief and plausibility} 

We say that random fuzzy set $\tX$ is \emph{normalized}  if it verifies the following conditions:
\be
\item For all $\omega\in\Omega$, $\tX(\omega)$ is either the empty set, or a normal fuzzy set, i.e., $\height(\tX(\omega))\in\{0,1\}$.
\item $P(\{\omega \in \Omega: \tX(\omega)=\emptyset\})=0$.
\end{enumerate}
These conditions will be assumed  in the rest of this section. For any $\omega\in\Omega$, let $\Pi_\tX(\cdot\mid \omega)$ be the possibility measure on $\Theta$ induced by $\tX(\omega)$:
\begin{equation}
\label{eq:defPi}
\Pi_\tX(B\mid \omega)=\sup_{\theta\in B} \tX(\omega)(\theta),
\end{equation}
and let  $N_\tX(\cdot\mid \omega)$ be the dual necessity measure:
\[
N_\tX(B\mid \omega)=\begin{cases}
1-\Pi_\tX(B^c\mid \omega) & \text{if } \tX(\omega)\neq \emptyset\\
0 &  \text{otherwise. }
\end{cases}
\]
Let  $Bel_\tX$ and $Pl_\tX$ be the mappings from $\sigma_\Theta$ to $[0,1]$ defined as
\begin{equation}
\label{eq:defBel}
Bel_\tX(B)=\int_\Omega  N(B\mid \omega) dP(\omega) 
\end{equation}
and
\begin{equation}
\label{eq:defPl}
Pl_\tX(B)=\int_\Omega  \Pi(B\mid \omega) dP(\omega). 
\end{equation}

Function $Bel_\tX$ is a belief function, and $Pl_\tX$ is the dual plausibility function. As shown in  \cite[Lemma 6.2]{couso11}, they are induced by the  random set $(\Omega\times [0,1], \sigma_\Omega\otimes\beta_{[0,1]}, P\otimes \lambda,\Theta,\sigma_\Theta,\bbX)$, where $\bbX: \Omega \times [0,1] \rightarrow 2^\Theta$ is the multi-valued mapping defined as 
\begin{equation}
\label{eq:Xbar}
\bbX(\omega,\alpha)= \cut{\alpha}{\tX}(\omega).
\end{equation}
As a consequence, $Bel_\tX(B)$ and $Pl_\tX(B)$ can also be written as follows:
\begin{subequations}
\label{eq:calcBelPl}
\begin{equation}
\label{eq:calcBel}
Bel_\tX(B)=\int_0^1 Bel_{\cut{\alpha}{\tX}}(B) d\alpha
\end{equation}
and
\begin{equation}
\label{eq:calcPl}
Pl_\tX(B)=\int_0^1 Pl_{\cut{\alpha}{\tX}}(B) d\alpha.
\end{equation}
\end{subequations}

%

\new{
\paragraph{Lower and upper expectations of a random fuzzy number} Let $\tX$ be a random fuzzy number (i.e., a random fuzzy set with domain $[0,1]^\reels$), and let $\bbX$ be the corresponding random set defined by \eqref{eq:Xbar}. We define the lower and upper expectations of $\tX$ as the lower and upper expectations of $\bbX$, i.e., $\esp_*(\tX)=\esp_*(\bbX)$ and $\esp^*(\tX)=\esp^*(\bbX)$. It follows from \eqref{eq:calcBelPl} that
\begin{equation}
\label{eq:calclowerupper}
\esp_*(\tX)=\int_0^1 \esp_*(\cut{\alpha}{\tX})d\alpha \quad \text{and} \quad \esp^*(\tX)=\int_0^1 \esp^*(\cut{\alpha}{\tX})d\alpha.
\end{equation}
}

\begin{Ex}
\label{ex:RFS_BelPl}
Let us consider again the random fuzzy number of Example \ref{ex:RFS}. Its lower and upper cdf's are, respectively, the mappings $x \rightarrow Bel_\tX((-\infty,x])$ and $x \rightarrow Pl_\tX((-\infty,x])$. Let us illustrate the calculation of the upper cdf first, using two methods.
\paragraph{Method 1} From \eqref{eq:defPi},
\[
\Pi\left((-\infty,x]\mid \omega\right)=\sup_{x'\le x}\tX(\omega)(x')=\begin{cases}
1 & \text{if } M(\omega)\le x\\
\frac{x-M(\omega)+a}{a} & \text{if } x < M(\omega) \le x+a\\
0 & \text{otherwise.}
\end{cases}
\]
Using \eqref{eq:defPl}, we get
\begin{align*}
Pl_\tX((-\infty,x]) &= P(M\le x)\times 1 + P(x < M \le x+a) \esp\left[\frac{x-M+a}{a} \mid x < M \le x+a\right]\\
&=\Phi\left(\frac{x-\mu}\sigma\right) + \left[\Phi\left(\frac{x+a-\mu}\sigma\right)-\Phi\left(\frac{x-\mu}\sigma\right)\right] \times\\
&\hspace{8cm} \left(\frac{x+a}a -\esp\left[M \mid x < M \le x+a\right] \right).
\end{align*}
Now, using a well-known result about the truncated normal distribution,
\[
\esp\left[M \mid x < M \le x+a\right]= \mu + \sigma \frac{\phi\left(\frac{x-\mu}\sigma\right)-\phi\left(\frac{x+a-\mu}\sigma\right)}{\Phi\left(\frac{x+a-\mu}\sigma\right)-\Phi\left(\frac{x-\mu}\sigma\right)}.
\]
After rearranging the terms, we  finally obtain
\begin{multline}
\label{eq:exPl}
Pl_\tX((-\infty,x])=\left(\frac{x+a-\mu}a\right)\Phi\left(\frac{x+a-\mu}\sigma\right) - \left(\frac{x-\mu}a\right)\Phi\left(\frac{x-\mu}\sigma\right) +\\
 \frac{\sigma}a \left[\phi\left(\frac{x+a-\mu}\sigma\right)-\phi\left(\frac{x-\mu}\sigma\right)\right].
\end{multline}

\paragraph{Method 2} Let  us now use \eqref{eq:calcPl}. We have
\begin{align*}
Pl_\tX((-\infty,x])&=\int_0^1 P(M-a(1-\alpha)\le x) d\alpha\\
&=\int_0^1 \Phi\left(\frac{x+a(1-\alpha)-\mu}\sigma\right) d\alpha.
\end{align*}
Using the formula
\[
\int \Phi(u+ v x)dx=\frac1v\left[(u+vx)\Phi(u+vx)+\phi(u+vx)\right] + C,
\]
we get the same result as \eqref{eq:exPl}. Using any of the two methods demonstrated above, we obtain the following expression for the lower cdf:
\begin{multline}
\label{eq:exBel}
Bel_\tX((-\infty,x])=\left(\frac{x-\mu}a\right)\Phi\left(\frac{x-\mu}\sigma\right) - \left(\frac{x-a-\mu}a\right)\Phi\left(\frac{x-a-\mu}\sigma\right) +\\
 \frac{\sigma}a \left[\phi\left(\frac{x-\mu}\sigma\right)-\phi\left(\frac{x-a-\mu}\sigma\right)\right].
\end{multline}
It can easily be checked that, when $a=0$, 
\[
Bel_\tX((-\infty,x])=Pl_\tX((-\infty,x])=\Phi\left(\frac{x-\mu}\sigma\right).
\]
Examples of functions $Bel_\tX((-\infty,x])$ and $Pl_\tX((-\infty,x])$ for different values of $a$ are shown in Figure \ref{fig:ex2}.

\begin{figure}
\centering  
\includegraphics[width=0.5\textwidth]{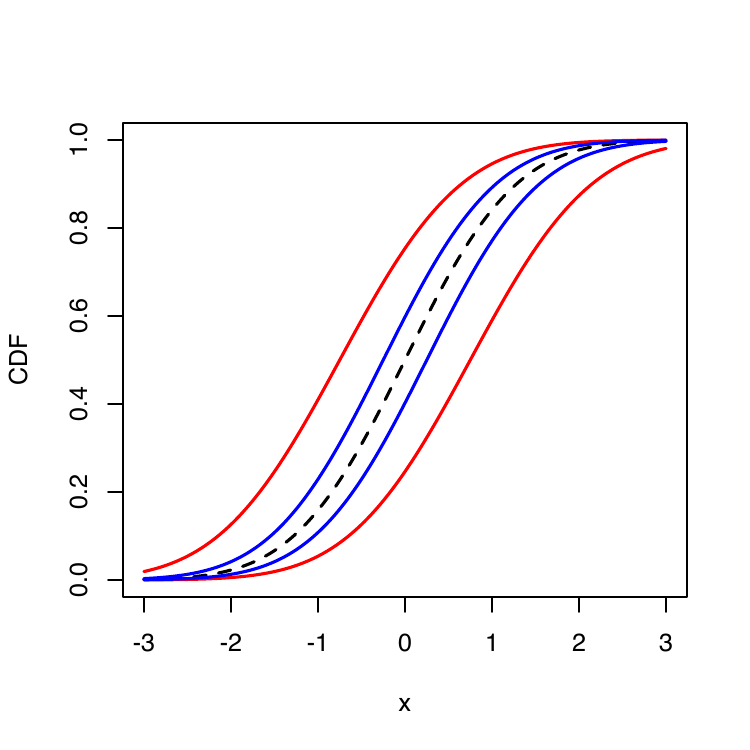}
\caption{Lower and upper cdf's for the random fuzzy numbers studied in Examples \ref{ex:RFS} and  \ref{ex:RFS_BelPl}, with $\mu=0$, $\sigma=1$, and $a=0.5$ (blue curves) or $a=1.5$ (red curves). The Gaussian cdf corresponding to $a=0$ is shown as a broken line. \label{fig:ex2}}
\end{figure}

\new{Now, the lower and upper expectations of $\tX$ can be computed from \eqref{eq:calclowerupper} as
\[
\esp_*(\tX)=\int_0^1 \esp_*(\cut{\alpha}{\tX})d\alpha=\int_0^1 [\mu - a(1-\alpha)] d\alpha=\mu-\frac{a}2,
\]
and
\[
\esp^*(\tX)=\int_0^1 \esp^*(\cut{\alpha}{\tX})d\alpha=\int_0^1 [\mu + a(1-\alpha)] d\alpha=\mu+\frac{a}2.
\]
}
\end{Ex}

\subsection{Generalized product-intersection rule} 
\label{subsec:genDS}

Dempster's rule and the possibilistic product intersection rule  recalled, respectively, in Sections \ref{subsec:random_sets} and \ref{subsec:fuzzy} can  be generalized to combine epistemic random fuzzy sets. Consider two  epistemic random fuzzy sets $(\Omega_1,\sigma_{1},P_1,\Theta, \sigma_\Theta,\tX_1)$ and $(\Omega_2,\sigma_2,P_2,\Theta, \sigma_\Theta,\tX_2)$ encoding independent pieces of evidence. The independence assumption means here that the relevant probability measure on the joint measurable space $(\Omega_1\times\Omega_2,\sigma_1\otimes\sigma_2$) is the product measure $P_1\times P_2$.

If  interpretations $\omega_1\in \Omega_1$ and $\omega_2\in \Omega_2$ both hold, we know that ``$\btheta \text{ is } \tX_1(\omega_1)$'' and ``$\btheta \text{ is } \tX_2(\omega_2)$''. It is then natural to combine the fuzzy sets $\tX_1(\omega_1)$ and $\tX_2(\omega_2)$ by an intersection operator. As discussed in Section \ref{subsec:fuzzy}, normalized product intersection is a good candidate as it suitable for combining fuzzy information from independent sources and it is associative. We will thus consider the mapping $\tX_\varodot(\omega_1,\omega_2)=\tX_1(\omega_1)\varodot \tX_2(\omega_2)$, which we will assume to be $\sigma_1\otimes\sigma_2$-$\sigma_\Theta$ strongly measurable.

As in the crisp case recalled in Section  \ref{subsec:random_sets}, if $\height(\tX_1(\omega_1)\tX_2(\omega_2))=0$, the two interpretations $\omega_1$ and $\omega_2$ are inconsistent and they must be discarded. If  $\height(\tX_1(\omega_1)\tX_2(\omega_2))=1$, the two interpretations are fully consistent. If $0<\height(\tX_1(\omega_1)\tX_2(\omega_2))<1$, $\omega_1$ and $\omega_2$ are \emph{partially consistent}. 
As proposed in \cite{denoeux21a},  instead of simply discarding only fully inconsistent pairs $(\omega_1,\omega_2)$, it makes sense to give all pairs  $(\omega_1,\omega_2)$ a weight proportional to the \emph{degree of consistency} between $\tX_1(\omega_1)$ and $\tX_2(\omega_2)$. This can be achieved by  conditioning $P_1\times P_2$ on the \emph{fuzzy set} $\tTheta^*$ of consistent pairs of interpretations, with membership function
\[
\tTheta^*(\omega_1,\omega_2)= \height\left(\tX_1(\omega_1)\cdot \tX_2(\omega_2)\right).
\]
Using Zadeh's definition of a fuzzy event \cite{zadeh68}, we get the following expression for the conditional probability measure $\tP_{12}=(P_1\times P_2)(\cdot \mid \tTheta^*)$,  for any $B \in \sigma_1\otimes\sigma_2$:
\[
\tP_{12}(B)=\frac{(P_1\times P_2)(B\cap \tTheta^*)}{(P_1\times P_2)(\tTheta^*)}=\frac{\int_{\Omega_1}\int_{\Omega_2}B(\omega_1,\omega_2) \height\left(\tX_1(\omega_1)\cdot \tX_2(\omega_2)\right) dP_2(\omega_2) dP_1(\omega_1)}{\int_{\Omega_1}\int_{\Omega_2} \height\left(\tX_1(\omega_1)\cdot \tX_2(\omega_2)\right) dP_2(\omega_2) dP_1(\omega_1)},
\]
where  $B(\cdot,\cdot)$ denotes the indicator function of $B$. This conditioning operation, called \emph{soft normalization} was first proposed in \cite{yager96} in the finite case and with a different justification. 

The combined  random fuzzy set 
\[
(\Omega_1\times\Omega_2,\sigma_1\otimes\sigma_2,\tP_{12},\Theta, \sigma_\Theta,\tX_\varodot)
\]
is called the \emph{orthogonal sum} of the two pieces of evidence. This operation generalizes both Dempster's rule and the normalized product of possibility distribution. We will refer to it as  the \emph{generalized product-intersection rule}, and  it will be denoted by the same symbol $\oplus$ as Dempster's rule.  It is clear that $\tX\oplus \bbX_0=\tX$ for any random fuzzy set $\tX$ and any vacuous random set $\bbX_0$ on the same domain $\Theta$. The degree of conflict between two random fuzzy sets $\tX_1$ and $\tX_2$ is naturally defined as 
\begin{equation}
\label{eq:conflict_fuzzy}
\kappa= 1-(P_1\times P_2)(\tTheta^*)=1-\int_{\Omega_1}\int_{\Omega_2} \height\left(\tX_1(\omega_1)\tX_2(\omega_2)\right) dP_2(\omega_2) dP_1(\omega_1).
\end{equation}

The associativity of  $\oplus$ was proved in \cite{denoeux21a} in the finite case; we give  a similar proof in the general case. 

\begin{Prop}
\label{prop:assoc_soft}
The generalized product-intersection rule $\oplus$ for random fuzzy sets is commutative and associative.
\end{Prop}
\begin{proof}
See  \ref{app:assoc_soft}
\end{proof}

The following proposition states that a counterpart of   Proposition \ref{prop:prodQ} is still valid when combining independent random fuzzy sets, i.e., the combined contour function is still proportional to the product of the contour functions.

\begin{Prop}
Let $\tX_1$ and $\tX_2$ be two random fuzzy sets on the same domain $\Theta$, with contour functions $pl_{\tX_1}$ and $pl_{\tX_2}$ and with degree of conflict $\kappa$ defined by \eqref{eq:conflict_fuzzy}. The  contour function  $pl_{\tX_1\oplus \tX_2}$ of $\tX_1\oplus \tX_2$ verifies
\begin{equation}
\label{eq:prodpl_fuzzy}
(pl_{\tX_1\oplus\tX_2})(\theta) =\frac{pl_{\tX_1}(\theta) pl_{\tX_2}(\theta)}{1-\kappa},
\end{equation}
for all $\theta\in \Theta$.
\end{Prop}
\begin{proof}
We have
\begin{align*}
(pl_{\tX_1\oplus \tX_2})(\theta) &=\frac{\int_{\Omega_1}\int_{\Omega_2} \height\left(\tX_1(\omega_1)\cdot\tX_2(\omega_2)\right) \tX_\varodot(\omega_1,\omega_2)(\theta) dP_2(\omega_2) dP_1(\omega_1)}{1-\kappa}\\
&=\frac{\int_{\Omega_1}\int_{\Omega_2} \height\left(\tX_1(\omega_1)\cdot\tX_2(\omega_2)\right) \frac{\tX_1(\omega_1)(\theta)\tX_2(\omega_2)(\theta)}{\height\left(\tX_1(\omega_1)\tX_2(\omega_2)\right)} dP_2(\omega_2) dP_1(\omega_1)}{1-\kappa}\\
&=\frac{\left(\int_{\Omega_1}\tX_1(\omega_1)(\theta) dP_1(\omega_1) \right)\left(\int_{\Omega_2}  \tX_2(\omega_2)(\theta) dP_2(\omega_2)\right)}{1-\kappa}\\
&=\frac{pl_{\tX_1}(\theta) pl_{\tX_2}(\theta)}{1-\kappa}.
\end{align*}
\end{proof}

As remarked in Section \ref{subsec:fuzzy}, a belief function induced by a random fuzzy set is also induced by a  random (crisp) set. However, combining   random fuzzy sets or   random crisp sets does not result in the same belief function in general. In particular, it is well-known that Dempster's rule does not preserve consonance. To combine two belief functions, we must, therefore, examine the evidence on which they are based, not only to determine whether the  bodies of evidence are independent or not, but also to determine whether the evidence is fuzzy or crisp. This point is illustrated by the following example.

\begin{Ex}
Consider the following two mappings from $\reels$ to $[0,1]$ represented in Figure \ref{fig:ex1_pdf}:
\[
\pi_1(x)=\GFN(0,0.3), \quad \pi_2(x)=\GFN\left(1,0.5\right).
\]
If  these two mappings are possibility distributions encoding fully reliable but fuzzy evidence, they correspond to ``constant random fuzzy sets'', i.e., mappings $\tX_1(\omega)=\pi_1$ and $\tX_2(\omega)=\pi_2$ with $P(\{\omega\})=1$. The combined random fuzzy set $\tX_1\oplus\tX_2$ is then defined by $(\tX_1\oplus\tX_2)(\omega)=\pi_1\varodot\pi_2$. From Proposition \ref{prop:prodphi}, the normalized product of two GFN's is a GFN. Here, we get the combined possibility distribution plotted as a red broken curve in Figure \ref{fig:ex1_pdf}:
\[
(\pi_1\varodot\pi_2)(x)=\GFN(0.625,0.8). 
\]
\begin{figure}
\centering  
\subfloat[\label{fig:ex1_pdf}]{\includegraphics[width=0.5\textwidth]{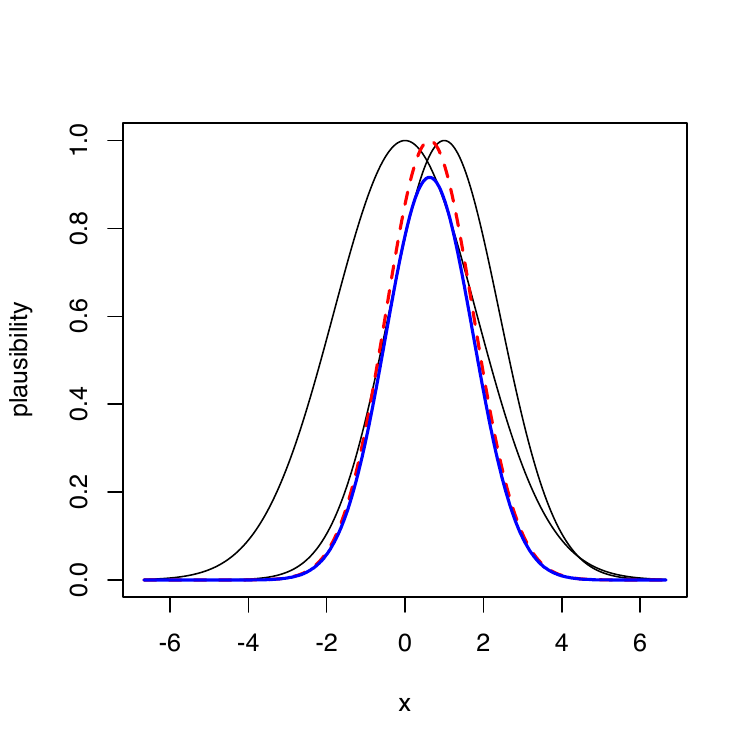}}
\subfloat[\label{fig:ex1_cdf}]{\includegraphics[width=0.5\textwidth]{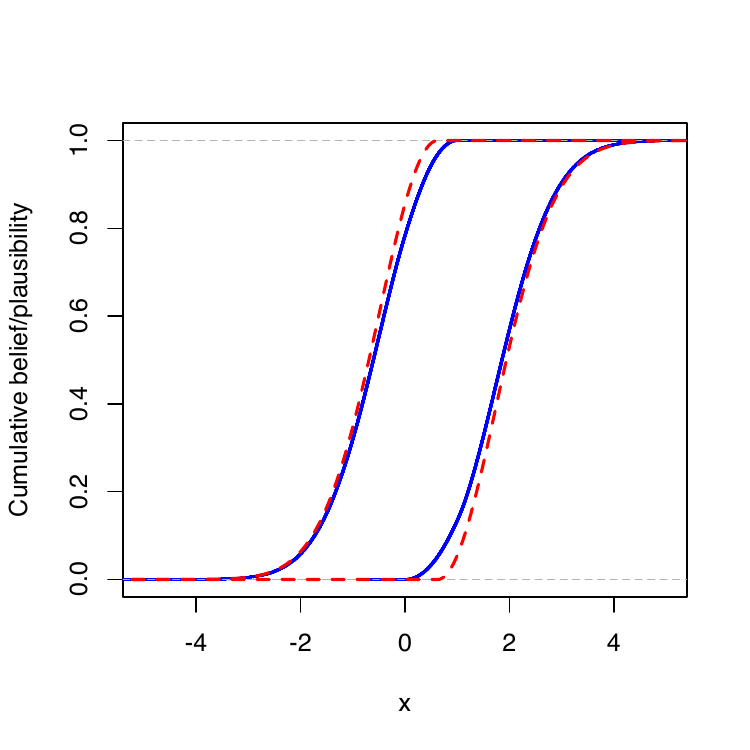}}
\caption{(a): Two Gaussian possibility distributions (black solid curves) with their normalized product intersection (red broken curve) and the contour function of the combined random set (blue solid curve). (b): Lower and upper cdf's of the combined possibility distribution (red broken curves) and of the combined random set (blue solid curves). \label{fig:ex1}}
\end{figure}
The corresponding lower and upper cumulative distribution functions (cdf's) are, respectively
\[
Bel_{\tX_1\oplus\tX_2}((-\infty,x])=\begin{cases}
0 & \text{if } x\le 0.625\\
1-\exp\left(-0.4 (x-0.625)^2\right) & \text{if } x> 0.625
\end{cases}
\]
and
\[
Pl_{\tX_1\oplus\tX_2}((-\infty,x])=\begin{cases}
\exp\left(-0.4 (x-0.625)^2\right) & \text{if } x\le 0.625\\
1 & \text{if } x>  0.625.
\end{cases}
\]

These two functions are plotted as red broken curves in Figure \ref{fig:ex1_cdf}. Alternatively, as explained in Section \ref{subsec:random_sets}, we may see $\pi_1$ and $\pi_2$ as encoding crisp but partially reliable evidence, in which case they define two independent consonant random intervals $\bbX_1(\alpha_1)=\cut{{\alpha_1}}{\pi_1}$ and $\bbX(\alpha_2)=\cut{{\alpha_2}}{\pi_2}$ , where $(\alpha_1,\alpha_2)$ has a uniform distribution on $[0,1]^2$. These two random intervals can be combined numerically using Monte-Carlo simulation, as explained in \cite{ann16}.  The contour function and the lower and upper cdf's are plotted as solid blue lines in Figures \ref{fig:ex1_pdf} and \ref{fig:ex1_cdf}, respectively. We notice that the contour functions are proportional, as a consequence of Proposition \ref{prop:prodQ}.
\end{Ex}

\subsection{Marginalization and vacuous extension}
\label{subsec:margin}

 Let us now consider again the case where we have two variables $\btheta_1$ and $\btheta_2$ with respective domains $\Theta_1$ and $\Theta_2$. Let  $\tX_{12}$ be  a random fuzzy set from a probability space $(\Omega,\sigma_\Omega,P)$ to  the measurable space $(\Theta_{12},\sigma_{\Theta_{12}})$ with $\Theta_{12}=\Theta_1\times\Theta_2$ and $\sigma_{\Theta_{12}}=\sigma_{\Theta_1}\otimes\sigma_{\Theta_2}$, where $\sigma_{\Theta_1}$ and $\sigma_{\Theta_2}$ are $\sigma$-algebras on $\Theta_1$ and $\Theta_2$, respectively. Let $\tX_1$ be the mapping from $\Omega$ to $[0,1]^{\Theta_1}$ defined by
\[
\tX_1(\omega)=\tX_{12}(\omega)\downarrow \Theta_1,
\]
where, as before, $\downarrow$ denotes fuzzy set projection. If, for all $\alpha\in [0,1]$, the mapping $\cut{\alpha}{\tX_1}$ is $\sigma_\Omega-\sigma_{\Theta_1}$ strongly measurable, then the random fuzzy set $\tX_1$ is called the \emph{marginal} of $\tX_{12}$ on $\Theta_1$. 

Conversely, given  a random fuzzy set $\tX_1$ from $(\Omega,\sigma_\Omega,P)$ to $(\Theta_{1},\sigma_{\Theta_{1}})$, let $\tX_{1\uparrow (1,2)}$ be the mapping from $\Omega$ to $[0,1]^{\Theta_{12}}$ that maps each $\omega \in \Omega$ to the cylindrical extension of $\tX_1(\omega)$ in $\Theta_{12}$
\[
\tX_{1\uparrow (1,2)}(\omega)= \tX_1(\omega)\times \Theta_2,
\]
i.e., for all $(\theta_1,\theta_2)\in\Theta_{12}$,
\[
\tX_{1\uparrow (1,2)}(\omega)(\theta_1,\theta_2)= \tX_1(\omega)(\theta_1).
\]
If the mapping $\tX_{1\uparrow (1,2)}$ is $\sigma_\Omega-\sigma_{\Theta_{12}}$ strongly measurable, then the random fuzzy set $\tX_{1\uparrow (1,2)}$ is called the \emph{vacuous extension} of $\tX_1$ in $\Theta_{12}$. 

We say that a joint random fuzzy set is \emph{noninteractive} if it is equal to the orthogonal sum of the vacuous extensions of its projections:
\[
\tX_{12}=\tX_{1\uparrow (1,2)}\oplus \tX_{2\uparrow (1,2)} \quad \text{denoted as} \quad \tX_{1}\oplus \tX_{2}.
\]

A particular kind of  noninteractive random fuzzy sets will be studied in Section \ref{subsec:marginGRFV}.

\subsection{Application to statistical inference}
\label{subsec:stat}

Epistemic random fuzzy sets naturally arise in the context of statistical inference. As proposed by Shafer \cite{shafer76} and formally justified in \cite{denoeux13b}\cite{denoeux14a}, the  information conveyed by the likelihood function in statistical inference problems can be represented by a consonant belief function, whose contour function is equal to the relative likelihood function. For a statistical model $f(\bx,\theta)$, where $\bx\in\calX$ is the observation and $\theta\in\Theta$ is the unknown parameter, the likelihood-based belief function $Bel(\cdot,\bx)$ on $\Theta$ after observing $\bx$ is, thus, consonant and  defined by the contour function 
\begin{equation}
\label{eq:lik}
pl(\theta;\bx)=\frac{L(\theta;\bx)}{\sup_{\theta'\in \Theta} L(\theta';\bx)},
\end{equation}
where $L(\cdot,\bx):\theta\rightarrow f(\bx;\theta)$ is the likelihood function, and it is assumed that the denominator in \eqref{eq:lik} is finite. The corresponding plausibility function is, thus, defined by
\[
Pl(A;\bx)=\sup_{\theta\in A} pl(\theta;\bx)
\]
for any $A\subset \Theta$, i.e., it is a possibility measure. However, as noticed by Shafer in \cite{shafer76} and \cite{shafer82}, and also discussed in \cite{denoeux13b}, this construction is not compatible with Dempster's rule: if we consider two independent observations $\bx$ and $\bx'$, the belief function $Bel(\cdot;\bx,\bx')$ is not equal to the orthogonal sum $Bel(\cdot;\bx)\oplus Bel(\cdot;\bx')$, which is not consonant. As argued in \cite{denoeux21a}, this problem disappears if we do not consider the likelihood-based belief function to be induced by a consonant random crisp set, but by a constant random fuzzy set $\ttheta_\bx$ with membership function $\ttheta_\bx(\theta)=pl(\theta;\bx)$. We can interpret $\ttheta_\bx$  as the fuzzy set of likely values of $\theta$ after observing $\bx$.
Combining the contour functions \eqref{eq:lik} by the normalized product intersection rule then yields the correct result, i.e., the constant random fuzzy set $\ttheta_{\bx,\bx'}$ with membership function $\ttheta_{\bx,\bx'}(\theta)=\ttheta_\bx(\theta)\varodot\ttheta_{\bx'}(\theta)$. 

Now, consider a prediction problem, where we want to predict the value of a random variable $Y$ whose distribution also depends on $\theta$. We can always write $Y$ in the form $Y=\varphi(\theta,U)$, where $U$ is a pivotal random variable with known distribution \cite{ann14a,ann16}. After observing the data $\bx$, our knowledge of $\theta$ is represented by the fuzzy set $\ttheta_\bx$. Conditionally on $U=u$, our knowledge of $Y$ is, thus, represented by the fuzzy set $\tY(u)=\varphi(\ttheta_\bx,u)$, with membership function
\[
\tY(u)(y)=\sup_{\theta:  \varphi(\theta,u)=y} \ttheta_\bx(\theta).
\]
The mapping $\tY: u \rightarrow \tY(u)$ is, then, a random fuzzy set representing statistical evidence about $Y$. 

\begin{Ex}
Let $\bX=(X_1,\ldots,X_n)$ be an independent and identically distributed (iid) Gaussian sample with parent distribution $N(\theta,1)$, and let $Y\sim N(\theta,1)$.  After observing a realization $\bx$ of $\bX$, the likelihood function is
\[
L(\theta;\bx)=(2\pi)^{-n/2}\exp\left(-\frac12 \sum_{i=1}^n(x_i-\theta)^2\right).
\]
Denoting by $\widehat{\theta}$  the sample mean, the fuzzy set $\ttheta_\bx$ of likely values of $\theta$ after observing $\bx$ is the relative likelihood
\[
\ttheta_\bx(\theta)=\frac{L(\theta;\bx)}{L(\widehat{\theta};\bx)}=\exp\left(-\frac{n}2 (\theta-\widehat{\theta})^2\right).
\]
It is the Gaussian fuzzy number $\GFN(\widehat{\theta},n)$ with mode $\widehat{\theta}$ and precision $n$. Now, $Y$ can be written as $Y=\theta+U$, with $U\sim N(0,1)$. Consequently, the conditional possibility distribution on $Y$ given $U=u$ is the Gaussian fuzzy number $\ttheta_\bx+u=\GFN(\widehat{\theta}+u,n)$, and  our knowledge of $Y$ is described by the random fuzzy set $U \rightarrow \GFN(\widehat{\theta}+U,n)$, with $U\sim N(0,1)$. This is a Gaussian fuzzy number with fixed precision $h=n$ and normal random mode $M=\widehat{\theta}+U\sim N(\widehat{\theta},1)$. This important class of random fuzzy sets will be studied in the next section.  
\end{Ex}


\section{Gaussian random fuzzy numbers}
\label{sec:GRFS}

In this section, we introduce Gaussian random fuzzy numbers (GRFN's) as a practical model for representing uncertainty on a real variable. As we will see, this model encompasses Gaussian random variables and Gaussian fuzzy numbers as special cases. A GRFN can be seen, equivalently,  as a Gaussian random variable with fuzzy mean, or as a Gaussian fuzzy number with random mode. The definition and main properties will first be presented in Section \ref{subsec:GRFN}. The expression of the orthogonal sum of two GRFN's will then be derived in Section \ref{subsec:GRFNcomb}. Finally, arithmetic operations on GRFN's will be addressed in Section \ref{subsec:arithm}.

\subsection{Definition and main properties} 
\label{subsec:GRFN}

\begin{Def}
Let $(\Omega,\sigma_\Omega,P)$ be a probability space and let $M: \Omega\rightarrow\reels$ be a Gaussian random variable with mean $\mu$ and variance $\sigma^2$. The random fuzzy set  $\tX:\Omega\rightarrow [0,1]^\reels$ defined as 
\[
\tX(\omega) = \GFN(M(\omega),h)
\]
is called a Gaussian random fuzzy number (GRFN) with mean $\mu$, variance $\sigma^2$ and precision $h$, which we write   $\tX\sim\tN(\mu,\sigma^2, h)$.
\end{Def}

\new{In the definition of a GRFN,  $\mu$ is a location parameter, while parameters $h$ and $\sigma^2$ correspond, respectively, to possibilistic and probabilistic uncertainty. If $h=0$, imprecision is maximal whatever the values of $\mu$ and $\sigma^2$: the GRFN $\tX$ then induces the vacuous belief function on $\reels$, in which case $Bel_\tX(A)=0$ for all $A\subset \reels$, and $Pl_\tX(A)=1$ for all $A\subseteq \reels$ such that $A\neq\emptyset$; such a  GRFN will be said to be \emph{vacuous} and  will be denoted by $\tX \sim \tN(0,1,0)$. If $h=+\infty$, each fuzzy number $\GFN(M(\omega),h)$ is reduced to a point: the  GRFN $\tX$ is then equivalent to a Gaussian random variable with mean $\mu$ and variance $\sigma^2$, which we can write: $\tN(\mu,\sigma^2,+\infty)=N( \mu,\sigma^2)$.} Another special case of interest is that where $\sigma^2=0$, in which case $M$ is a constant random variable taking value $\mu$, and $\tX$ is a possibilistic variable with possibility distribution $\GFN(\mu,h)$.

The following proposition gives the expression of the contour functions $pl_\tX(x)$ associated to $\tX$. 


\begin{Prop} 
\label{prop:contourRFS}
The contour function of GRFN $\tX\sim \tN(\mu,\sigma^2,h)$  is
\begin{equation}
\label{eq:pl}
pl_\tX(x)=\frac{1}{\sqrt{1+h\sigma^2}}\exp\left(- \frac{h(x-\mu)^2}{2(1+h\sigma^2)}\right).
\end{equation}
\end{Prop}
\begin{proof}
See \ref{app:contourRFS}.
\end{proof}

A shown by Proposition \ref{prop:contourRFS}, the contour function $pl_\tX$ is constant in two cases:  if $h=0$, $\tX$ is vacuous, and $pl_\tX(x)=1$ for all $x\in\reels$; if $h=+\infty$, $\tX$ is a random variable,  and $pl_\tX(x)=0$ for all $x\in\reels$. We also note that, if $\sigma^2=0$, $pl_\tX$ is equal to the possibility distribution $\GFN(\mu,h)$. \new{When $\sigma^2\rightarrow+\infty$ and $h>0$, $pl_\tX(x)\rightarrow 0$ for all $x$.} The next proposition gives the expressions of the belief and plausibility of any real interval.

\begin{Prop}
\label{prop:BelPlRFS}
For any real interval $[x,y]$, the degrees of   belief and plausibility   of $[x,y]$ induced by the GRFN $\tX\sim \tN(\mu,\sigma^2,h)$ are, respectively,
\begin{multline}
\label{eq:belint}
Bel_\tX([x,y])=\Phi\fracpar{y-\mu}{\sigma} -\Phi\fracpar{x-\mu}{\sigma} - \\
pl_\tX(x)\left[\Phi\fracpar{(x+y)/2-\mu+(y-x)h\sigma^2/2}{\sigma\sqrt{h\sigma^2+1}}-\Phi\fracpar{x-\mu}{\sigma\sqrt{h\sigma^2+1}} \right]- \\
pl_\tX(y)\left[\Phi\fracpar{y-\mu}{\sigma\sqrt{h\sigma^2+1}}-\Phi\fracpar{(x+y)/2-\mu-(y-x)h\sigma^2/2}{\sigma\sqrt{h\sigma^2+1}}\right],
\end{multline}
and
\begin{multline}
\label{eq:plint}
Pl_\tX([x,y])=\Phi\fracpar{y-\mu}{\sigma} -\Phi\fracpar{x-\mu}{\sigma} + pl_\tX(x)\Phi\fracpar{x-\mu}{\sigma\sqrt{h\sigma^2+1}} + \\
pl_\tX(y)\left[1-\Phi\fracpar{y-\mu}{\sigma\sqrt{h\sigma^2+1}}\right].
\end{multline}
\end{Prop}
\begin{proof}
See \ref{app:BelPlRFS}.
\end{proof}

\begin{Cor}
The lower and upper cdf's of the GRFN  $\tX\sim \tN(\mu,\sigma^2,h)$  are, respectively
\begin{equation}
\label{eq:lowercdf}
Bel_\tX((-\infty,y])=\Phi\fracpar{y-\mu}{\sigma}  - pl_\tX(y) \Phi\fracpar{y-\mu}{\sigma\sqrt{h\sigma^2+1}}
\end{equation}
and
\begin{equation}
\label{eq:uppercdf}
Pl_\tX((-\infty,y])=\Phi\fracpar{y-\mu}{\sigma} + pl_\tX(y)\left[1-\Phi\fracpar{y-\mu}{\sigma\sqrt{h\sigma^2+1}}\right].
\end{equation}
\end{Cor}
\begin{proof}
Immediate from Proposition \ref{prop:BelPlRFS} by letting $x$ tend to $-\infty$ in \eqref{eq:belint} and  \eqref{eq:plint}
\end{proof}

\new{We can easily check from \eqref{eq:belint} and \eqref{eq:plint} that  $Bel_\tX([x,y])$ and $Pl_\tX([x,y])$ both tend to $\Phi\fracpar{y-\mu}{\sigma} -\Phi\fracpar{x-\mu}{\sigma}$ when $h\rightarrow \infty$, which is consistent with the fact that a GRFN with infinite precision is equivalent to a Gaussian random variable.} Finally, the following proposition gives the expressions of the lower and upper expectations of a GRFN.

\begin{Prop}
\label{prop:expect}
Let $\tX\sim\tN(\mu,\sigma^2,h)$ be a GRFN with $h>0$. Its lower and upper expectations are, respectively,
\begin{equation}
\label{eq:luexp}
\esp_*(\tX)=\mu-\sqrt{\frac{\pi}{2h}} \quad \text{and} \quad \esp^*(\tX)=\mu+\sqrt{\frac{\pi}{2h}}.
\end{equation}
\end{Prop}
\begin{proof}
See \ref{app:expect}.
\end{proof}

As expected, we can see from \eqref{eq:luexp} that the lower and upper expectations boil down to the usual expectation $\mu$ when $h=+\infty$.

\subsection{Orthogonal sum of Gaussian  random fuzzy numbers} 
\label{subsec:GRFNcomb}

In this section, we derive the expression of the orthogonal sum $\tX_1\oplus\tX_2$ of two GRFN's $\tX_1$ and $\tX_2$. We start with the following lemma.

\begin{Lem}
\label{lem:Gauss}
Let   $M_1\sim N(\mu_1,\sigma_1^2)$ and $M_2\sim N(\mu_2,\sigma_2^2)$ be two independent Gaussian random variables, and let $\tF$ be the fuzzy subset of $\reels^2$ with membership function
\[
\tF(m_1,m_2)= \height\left(\GFN(m_1,h_1) \cdot \GFN(m_2,h_2)\right)=\exp\left(-\frac{h_1h_2(m_1-m_2)^2}{2(h_1+h_2)}\right).
\] 
The conditional probability distribution of $(M_1,M_2)$ given $\tF$ is two-dimensional Gaussian with mean vector $\tbmu=(\tmu_1,\tmu_2)^T$ and covariance matrix 
\[
\tbSigma=\begin{pmatrix}
\tsigma_1^2 & \rho\tsigma_1\tsigma_2\\
 \rho\tsigma_1\tsigma_2 &\tsigma_2^2
\end{pmatrix},
\]
with
\begin{subequations}
\label{eq:parG}
\begin{align}
\label{eq:parGmu1}
\tmu_1 &=  \frac{\mu_1(1+\oh\sigma_2^2)+  \mu_2\oh\sigma_1^2}{1+\oh(\sigma_1^2+\sigma_2^2)} \\
\label{eq:parGmu2}
\tmu_2 &=  \frac{\mu_2(1+\oh\sigma_1^2)+\mu_1\oh\sigma_2^2}{1+\oh(\sigma_1^2+\sigma_2^2)} \\
\label{eq:parGsig1}
\tsigma_1^2 &= \frac{\sigma_1^2(1+\oh\sigma_2^2)}{1+\oh(\sigma_1^2+\sigma_2^2)}\\
\label{eq:parGsig2}
\tsigma_2^2 &= \frac{\sigma_2^2(1+\oh\sigma_1^2)}{1+\oh(\sigma_1^2+\sigma_2^2)}\\
\label{eq:parGrho}
\rho & = \frac{\oh \sigma_1\sigma_2}{\sqrt{(1+\oh\sigma_1^2)(1+\oh\sigma_2^2)}},
\end{align}
where
\begin{equation}
\label{eq:hbar}
\oh = \frac{h_1h_2}{h_1+h_2}.
\end{equation}
\end{subequations}
Furthermore,  the degree of conflict between two independent GRFN's $\tX_1\sim \tN(\mu_1,\sigma_1^2,h_1)$ and $\tX_2\sim \tN(\mu_2,\sigma_2^2,h_2)$ is
\begin{multline*}
\kappa=1- \iint f(m_1,m_2) \tF(m_1,m_2) dm_1dm_2 =\\
1-\frac{\tsigma_1\tsigma_2}{\sigma_1\sigma_2}\sqrt{1-\rho^2}\exp\left\{-\frac12 \left[\frac{\mu_1^2}{\sigma_1^2}+ \frac{\mu_2^2}{\sigma_2^2}\right] +\frac1{2(1-\rho^2)} \left[\frac{\tmu_1^2}{\tsigma_1^2}+ \frac{\tmu_2^2}{\tsigma_2^2} - 2\rho \frac{\tmu_1\tmu_2}{\tsigma_1\tsigma_2}\right]\right\},
\end{multline*}
where $f(m_1,m_2)$ is the pdf of random vector $(M_1,M_2)$.
\end{Lem}
\begin{proof}
See \ref{app:Gauss}.
\end{proof}

\begin{Prop}
\label{prop:combGRFN}
Let  $\tX_1\sim\tN(\mu_1,\sigma_1^2,h_1)$ and $\tX_2\sim \tN(\mu_2,\sigma_2^2,h_2)$ be two independent GRFN's, and assume that $h_1>0$ or $h_2>0$. We have
\[
\tX_1 \oplus \tX_2 \sim \tN(\tmu_{12},\tsigma_{12}^2,h_{12}),
\]
with 
\begin{equation}
\label{eq:h12}
h_{12}=h_1+h_2,
\end{equation}
\begin{equation}
\label{eq:mu12}
\tmu_{12}=\frac{h_1 \tmu_1+h_2 \tmu_2}{h_1+h_2},
\end{equation}
and
\begin{equation}
\label{eq:sig12}
\tsigma_{12}^2=\frac{h^2_1 \tsigma^2_1+h^2_2 \tsigma^2_2 + 2 \rho h_1 h_2 \tsigma_1\tsigma_2}{(h_1+h_2)^2},
\end{equation}
where $\tmu_1$, $\tmu_2$, $\tsigma_1$, $\tsigma_2$ and $\rho$ are given by \eqref{eq:parG} in Lemma \ref{lem:Gauss}.
\end{Prop}
\begin{proof}
Let $M_1$ and $M_2$ be the Gaussian random variables from $(\Omega_1,\sigma_1,P_1)$ and   $(\Omega_2,\sigma_2,P_2)$ to $(\reels,\beta_\reels)$ corresponding, respectively, to GRFN's $\tX_1\sim \tN(\mu_1,\sigma_1^2,h_1)$ and $\tX_2\sim \tN(\mu_2,\sigma_2^2,h_2)$. The orthogonal sum of $\tX_1$ and $\tX_2$ is the random fuzzy set $(\Omega_1\times\Omega_2,\sigma_1\otimes \sigma_2,\tP_{12}, \reels,\beta_\reels,\tX_\varodot)$, where $\tX_\varodot$ is the mapping
\[
\tX_\varodot: (\omega_1,\omega_2) \rightarrow \GFN(M_{12}(\omega_1,\omega_2),h_1+h_2),
\]
with 
\[
M_{12}(\omega_1,\omega_2)=\frac{h_1M_1(\omega_1)+h_2M_2(\omega_2)}{h_1+h_2},
\]
and $\tP_{12}$ is the probability measure on $\Omega_1\times\Omega_2$ obtained by conditioning $P_1\times P_2$ on the fuzzy set $\tTheta^*(\omega_1,\omega_2)=\height\left(\GFN(M_{1}(\omega_1),h_1),\GFN(M_{2}(\omega_2),h_2)\right)$. From Lemma \ref{lem:Gauss}, the pushforward measure of $\tP_{12}$ by the random vector  $(M_1,M_2)$ is the two-dimensional Gaussian distribution with parameters $(\tmu_1,\tmu_2,\tsigma_1,\tsigma_2,\rho)$. Consequently, $M_{12}$ is a Gaussian random variable with mean 
\[
\esp(M_{12})=\frac{h_1 \esp(M_1)+h_2 \esp(M_2)}{h_1+h_2}=\frac{h_1 \tmu_1+h_2 \tmu_2}{h_1+h_2},
\]
and variance
\begin{align*}
\var(M_{12})&=\frac{h_1^2 \var(M_1)+h_2^2 \var(M_2)+2h_1h_2\cov(M_1,M_2)}{(h_1+h_2)^2}\\
&=\frac{h^2_1 \tsigma^2_1+h^2_2 \tsigma^2_2 + 2 \rho h_1 h_2 \tsigma_1\tsigma_2}{(h_1+h_2)^2}.
\end{align*}
\end{proof}

Let us now consider some special cases in which one of two GRFN's is a Gaussian random variable. The next proposition states that the orthogonal sum of a Gaussian random variable and an arbitrary GRFN with finite precision is a Gaussian random variable.

\begin{Prop} 
\label{prop:combRV}
Let  $X_1\sim N(\mu_1,\sigma_1^2)$ be a Gaussian random variable  and  $\tX_2\sim\tN(\mu_2,\sigma_2^2,h_2)$ a GRFN with finite precision $h_2<+\infty$. Their orthogonal sum is a Gaussian random variable  $X_1\oplus\tX_{2}\sim N(\tmu_{12},\tsigma_{12}^2)$ 
with
\begin{equation}
\label{eq:mu12bis}
\tmu_{12} = \frac{\mu_1(1+h_2\sigma_2^2)+  \mu_2 h_2 \sigma_1^2}{1+h_2 (\sigma_1^2+\sigma_2^2)},
\end{equation}
\begin{equation}
\label{eq:sig12bis}
\tsigma_{12}^2= \frac{\sigma_1^2(1+h_2 \sigma_2^2)}{1+h_2 (\sigma_1^2+\sigma_2^2)},
\end{equation}
and the probability density of $X_1\oplus\tX_2$ is proportional to the product of the pdf of $X_1$ and the contour function of $\tX_2$.
\end{Prop}
\begin{proof}
See \ref{app:combRV}.
\end{proof}

The following corollary addresses the special case where   $\tX_2$ is a possibilistic GRFN.
\begin{Cor}
Let  $X_1\sim N(\mu_1,\sigma_1^2)$ be a Gaussian random variable  and  $\tX_2\sim\tN(\mu_2,0,h_2)$ a possibilistic GRFN. Their orthogonal sum $X_1\oplus \tX_2$ is a Gaussian random variable and its distribution is the conditional distribution of $X_1$ given the fuzzy event $\GFN(\mu_2,h_2)$.
\end{Cor}
\begin{proof}
From Proposition \ref{prop:combRV}, $X_1\oplus \tX_2 \sim \tN(\tmu_{12},\tsigma_{12}^2)$ with
\[
\tmu_{12} = \frac{\mu_1+  \mu_2 h_2 \sigma_1^2}{1+h_2 \sigma_1^2} \quad \text{and} \quad 
\tsigma_{12}^2= \frac{\sigma_1^2}{1+h_2 \sigma_1^2}.
\]
Now, we know from Proposition \ref{prop:combRV} that the density of $X_1\oplus \tX_2$ is proportional to the product of the density of $X_1$ and the contour function of $\tX_2$, which is $\varphi(x;\mu_2,h_2)$. Consequently, we have
\[
f_{X_1\oplus \tX_2}(x)=\frac{\frac{1}{\sigma_1^2} \exp\left(-\frac12 \frac{(x-\mu_1)^2}{\sigma_1^2}\right)\exp\left(- \frac{h_2(x-\mu_2)^2}{2(1+h_2\sigma_2^2)}\right)}{\int \frac1{\sigma_1^2} \exp\left(-\frac12 \frac{(x-\mu_1)^2}{\sigma_1^2}\right)\exp\left(- \frac{h_2(x-\mu_2)^2}{2(1+h_2\sigma_2^2)} \right)dx},
\]
which is the conditional density $f_{X_1}(x|\GFN(\mu_2,h_2))$.
\end{proof}

Finally, another special case of interest is when both GRFN's are Gaussian random variables. This case is addressed by the following corollary.

\begin{Cor}
Let $X_1\sim N(\mu_1,\sigma_1^2)$ and $X_2\sim N(\mu_2,\sigma_2^2)$ be two Gaussian random variables. We have $X_1\oplus X_2\sim N(\tmu_{12},\sigma_{12}^2)$ with
\begin{equation*}
\tmu_{12} = \frac{\mu_1\sigma_2^2+  \mu_2 \sigma_1^2}{\sigma_1^2+\sigma_2^2} \quad \text{and} \quad 
\tsigma_{12}^2= \frac{\sigma_1^2\sigma_2^2}{\sigma_1^2+\sigma_2^2}.
\end{equation*}
\end{Cor}
\begin{proof}
Immediate from Proposition \ref{prop:combRV} by letting $h_2$ tend to $+\infty$ in \eqref{eq:mu12bis} and \eqref{eq:sig12bis}.
\end{proof}

\subsection{Arithmetic operations on GRFN's} 
\label{subsec:arithm}

Arithmetic operations can be extended to fuzzy numbers using Zadeh's extension principles \cite{dubois80,dubois00b}. More precisely, let $\tA$ and $\tB$ be two fuzzy numbers, and let $*$ be a binary operation on reals. Then the fuzzy number $\tC=\tA * \tB$ is defined as
\[
\tC(c)=\sup_{c=a*b} \min(\tA(a),\tB(b)).
\]
The membership function $\tC$ is equal to the possibility distribution on $c=a*b$, if $a$ and $b$ are constrained, respectively, by possibility distributions $\tA$ and $\tB$. Unary or $n$-ary operations can be extended from real to fuzzy numbers in the same way. For a certain class of fuzzy number called LR-fuzzy numbers \cite[page 54]{dubois80}, closed-form expressions exist for the addition,  subtraction and scalar multiplication of fuzzy numbers. In particular, Gaussian fuzzy numbers with positive precision are LR fuzzy numbers and they verify the following equalities \cite{nahmias78}:
\begin{align*}
\GFN(m_1,h_1) + \GFN(m_2,h_2)&=\GFN(m_1+m_2,(h_1^{-1/2}+h_2^{-1/2})^{-2})\\
\GFN(m_1,h_1) - \GFN(m_2,h_2)&=\GFN(m_1-m_2,(h_1^{-1/2}+h_2^{-1/2})^{-2})\\
\lambda\cdot \GFN(m,h) &= \GFN(\lambda m, h/\lambda^2), \quad \forall \lambda\in \reels.
\end{align*}
As addition of fuzzy numbers is  associative, we can express the linear combination of $n$ GFN's as
\begin{equation}
\label{eq:lincomb}
\sum_{i=1}^n \lambda_i\cdot\GFN(m_i,h_i)=\GFN\left(\sum_{i=1}^n \lambda_i m_i,\left(\sum_{i=1}^n |\lambda_i|h_i^{-1/2}\right)^{-2}\right).
\end{equation}
Now, let us consider $n$ independent GRFN's $\tX_i$ from probability spaces $(\Omega_i,\sigma_i,P_i)$ to $[0,1]^\reels$ defined by
\[
\tX_i(\omega)=\GFN(M_i(\omega),h_i)
\]
for all $\omega\in\Omega_i$, where $M_i$ is a Gaussian random variable with mean $\mu_i$ and standard deviation $\sigma_i$, and $h_i>0$. Let 
\[
\tX=\sum_{i=1}^n \lambda_i \tX_i
\]
be the random fuzzy set from $(\Omega_1\times\ldots\times\Omega_n,\sigma_1\otimes\ldots\otimes\sigma_n,P_1\times\ldots\times P_n)$ to $[0,1]^\reels$ defined by
\[
\tX(\omega_1,\ldots,\omega_n)=\sum_{i=1}^n \lambda_i \cdot \GFN(M_i(\omega_i),h_i).
\]
If each GRFN $\tX_i$ represents our knowledge about the value of some quantity $X_i$, $\tX$ represents our knowledge about $X=\sum_{i=1}^n \lambda_i X_i$. From \eqref{eq:lincomb}, $\tX\sim\tN(\mu,\sigma,h)$ with
\[
\mu=\sum_{i=1}^n \lambda_i\mu_i, \quad \sigma^2=\sum_{i=1}^n \lambda_i^2\sigma_i^2, \quad \text{and} \quad  h=\left(\sum_{i=1}^n |\lambda_i|h_i^{-1/2}\right)^{-2}.
\]

 
\section{Gaussian random fuzzy vectors}
\label{sec:GRFV}

In this section, we introduce Gaussian random fuzzy vectors (GRFV's), an extension of the model presented in Section \ref{sec:GRFS} allowing us to describe knowledge about multidimensional quantities. The main definitions and properties are first  introduced in Section \ref{subsec:GRFVdef}. The expression of the orthogonal sum of two GRFV's is then  given in Section \ref{subsec:GRFVcomb}, after which the marginalization and vacuous extension of GRFV's are described in Section \ref{subsec:marginGRFV}. \new{Finally, our model is compared to Dempster's normal belief function model in Section \ref{subsec:dempster}.}

\subsection{Definition and main properties}
\label{subsec:GRFVdef}

We consider a $p$-dimensional variable $\btheta$ taking values in $\reels^p$. Knowledge about $\btheta$ may be encoded as a $p$-dimensional Gaussian fuzzy vector, defined as follows.

\begin{Def}
We define the $p$-dimensional Gaussian fuzzy vector (GFV) with center $\bm\in\reels^p$ and $p\times p$ symmetric and positive semidefinite precision matrix $\bH$ as the normalized fuzzy subset of $\reels^p$ with membership function
\[
\varphi(\bx;\bm,\bH)=\exp\left(-\frac12 (\bx-\bm)^T \bH (\bx-\bm)\right),
\]
denoted as $\GFV(\bm,\bH)$.
\end{Def}

As shown in \cite{petersen12}, the normalized product of two GFV's is still a GFV. The following proposition generalizes Proposition \ref{prop:prodphi}.

\begin{Prop}
\label{prop:prodphi1}
Let $\GFV(\bm_1,\bH_1)$ and $\GFV(\bm_2,\bH_2)$ be two $p$-dimensional GFV's with positive definite precision matrices $\bH_1$ and $\bH_2$. We have
\begin{multline*}
\varphi(\bx;\bm_1,\bH_1) \cdot \varphi(\bx;\bm_2,\bH_2)= \varphi(\bx;\bm_{12},\bH_{12}) \times \\
\exp\left(-\frac12 (\bm_1-\bm_2)^T(\bH_1^{-1}+ \bH_2^{-1})^{-1}(\bm_1-\bm_2)\right),
\end{multline*}
with 
\[
 \bm_{12}=(\bH_1+\bH_2)^{-1} (\bH_1 \bm_1+\bH_2 \bm_2) \quad \text{and} \quad \bH_{12}=\bH_1+\bH_2  .
\]
Consequently, the  following equation holds:
\[
\GFV(\bm_1,\bH_1)   \varodot \GFV(\bm_2,\bH_2)= \GFV(\bm_{12},\bH_{12}),
\]
and the height of the product intersection between $\GFV(\bm_1,\bH_1)$ and $\GFV(\bm_1,\bH_2)$ is
\begin{subequations}
\label{eq:etapD}
\begin{align}
\height\left(\GFV(\bm_1,\bH_1),\GFV(\bm_1,\bH_2)\right)&=\max_\bx \varphi(\bx;\bm_1,\bH_1)\varphi(\bx;\bm_2,\bH_2)\\
&=\exp\left(-\frac12 (\bm_1-\bm_2)^T(\bH_1^{-1}+ \bH_2^{-1})^{-1}(\bm_1-\bm_2)\right).
\end{align}
\end{subequations}
\end{Prop}

Equipped with the notion of GFV, we can now introduce a model of random fuzzy set  that can be seen as a GFV whose mode is a multidimensional Gaussian random variable. This model is defined formally as follows.

\begin{Def}
Let $(\Omega,\sigma_\Omega,P)$ be a probability space, $\bM: \Omega\rightarrow\reels^p$  a $p$-dimensional Gaussian random vector with mean $\bmu$ and variance matrix $\bSigma$, and $\bH$ a $p\times p$ symmetric and positive semidefinite real matrix. The random fuzzy set  $\tX:\Omega\rightarrow [0,1]^{\reels^p}$ defined as 
\[
\tX(\omega) = \GFV(\bM(\omega),\bH)
\]
is called a Gaussian random fuzzy vector (GRFV), which we denote as $\tX\sim \tN(\bmu,\bSigma,\bH)$.
\end{Def}

The following proposition generalizes Proposition \ref{prop:contourRFS}.

\begin{Prop} 
\label{prop:contourRFS1}
The contour function of GRFV $\tX\sim \tN(\bmu,\bSigma,\bH)$ with positive definite precision matrix $\bH$ is
\[
pl_\tX(\bx)=\frac{1}{|\bI_p+\bSigma \bH|^{1/2}}\exp\left(- \frac12 (\bx-\bmu)^T(\bH^{-1}+\bSigma)^{-1}(\bx-\bmu)\right),
\]
where $\bI_p$ is the $p$-dimensional identity matrix.
\end{Prop}
\begin{proof}
See \ref{app:contourRFS1}.
\end{proof}

\subsection{Orthogonal sum of Gaussian random fuzzy vectors} 
\label{subsec:GRFVcomb}

The practical interest of GRFV's arises from the fact that they can be easily combined by the generalized product-intersection rule. The following lemma and proposition, which generalize, respectively, Lemma \ref{lem:Gauss} and Proposition \ref{prop:combGRFN}, give us the expression of the orthogonal sum of two GRFV's.

\begin{Lem}
\label{lem:Gauss1}
Let   $\bM_1\sim\calN(\bmu_1,\bSigma_1)$ and $\bM_2\sim\calN(\bmu_2,\bSigma_2)$ be two independent Gaussian $p$-dimensional random vectors and let $\bH_1$ and $\bH_2$ be two symmetric and  positive definite  $p\times p$ matrices.  Let $\tF$ be the fuzzy subset of $\reels^{2p}$ with membership function
\[
\tF(\bm_1,\bm_2)= \height\left(\GFV(\bm_1,\bH_1) \cdot \GFV(\bm_2,\bH_2)\right),
\] 
and let $\bM$ be the $2p$-dimensional vector $(\bM_1,\bM_2)$. The conditional probability distribution of $\bM$ given $\tF$ is $2p$-dimensional Gaussian with mean vector $\tbmu$ and covariance matrix $\tbSigma$ defined as follows: 
\[
\tbSigma=\begin{pmatrix}
\bSigma_1^{-1}+\obH& -\obH\\
 -\obH &\bSigma_2^{-1}+\obH
\end{pmatrix}^{-1},
\]
\begin{equation}
\label{eq:tbmu}
\tbmu=\begin{pmatrix}
\obH^{-1}\bSigma_1^{-1}+\bI_p& -\bI_p\\
 -\bI_p &\obH^{-1}\bSigma_2^{-1}+\bI_p
\end{pmatrix}^{-1}
\begin{pmatrix}
\obH^{-1}\bSigma_1^{-1}& \bzero\\
 \bzero &\obH^{-1}\bSigma_2^{-1}
\end{pmatrix}
\begin{pmatrix}
\bmu_1\\
\bmu_2
\end{pmatrix},
\end{equation}
with
\[
\obH=(\bH_1^{-1}+\bH_2^{-1})^{-1}.
\]
Furthermore, the degree of conflict between two GRFV's $\tX_1\sim\tN(\bmu_1,\bSigma_1,\bH_1)$ and $\tX_2\sim\tN(\bmu_2,\bSigma_2,\bH_2)$ is
\begin{multline*}
\kappa=1- \int_{\reels^{2p}} f(\bm_1,\bm_2) \tF(\bm_1,\bm_2) d\bm_1d\bm_2=\\
1-\sqrt{\frac{|\tbSigma|}{|\bSigma_1|  |\bSigma_2|}} \exp\left\{-\frac12 \left[ \bmu_1^T \bSigma_1^{-1} \bmu_1+ \bmu_2^T \bSigma_2^{-1} \bmu_2 - \tbmu^T \tbSigma^{-1} \tbmu \right]\right\}.
\end{multline*}
\end{Lem}
\begin{proof}
See \ref{app:Gauss1}
\end{proof}

\begin{Prop}
Let  $\tX_1\sim\tN(\bmu_1,\bSigma_1,\bH_1)$ and $\tX_2\sim\tN(\bmu_2,\bSigma_2,\bH_2)$ be two independent GRFV's. We have
\[
\tX_1 \oplus \tX_2\sim \tN(\tbmu_{12},\tbSigma_{12},\bH_{12})
\]
with 
\[
\bH_{12}=\bH_1+\bH_2,
\]
\[
\tbmu_{12}=\bA \tbmu,
\]
and
\[
\tbSigma_{12}=\bA \tbSigma \bA^T,
\]
where $\bA$ is the constant $p\times 2p$ matrix defined as
\[
\bA=\bH_{12}^{-1}\begin{pmatrix}\bH_1 & \bH_2\end{pmatrix}.
\]
\end{Prop}
\begin{proof}
Let $\bM_1$ and $\bM_2$ be the Gaussian random vector from $(\Omega_1,\sigma_1,P_1)$ and   $(\Omega_2,\sigma_2,P_2)$ to $(\reels^p,\beta_{\reels^p})$ corresponding, respectively, to GRFV's $\tX_1\sim\tN(\bmu_1,\bSigma_1,\bH_1)$ and $\tX_2\sim\tN(\bmu_2,\bSigma_2,\bH_2)$. The orthogonal sum of $\tX_1$ and $\tX_2$ is defined by the mapping
\[
\tX_\varodot: (\omega_1,\omega_2) \rightarrow \GFV(\bM_{12}(\omega_1,\omega_2),\bH_1+\bH_2)
\]
with 
\[
 \bM_{12}=(\bH_1+\bH_2)^{-1} (\bH_1 \bM_1+\bH_2 \bM_2)=\bA\begin{pmatrix} \bM_1\\\bM_2\end{pmatrix},
\]
where $\bA$ is the $p\times 2p$ matrix
\[
\bA=(\bH_1+\bH_2)^{-1}\begin{pmatrix} \bH_1 & \bH_2\end{pmatrix},
\]
and  the probability measure $\tP_{12}$  on $\Omega_1\times\Omega_2$ obtained by conditioning $P_1\times P_2$ on the fuzzy set $\tTheta^*(\omega_1,\omega_2)=\height\left(\GFV(\bM_{1}(\omega_1),\bH_1),\GFV(\bM_{2}(\omega_2),\bH_2)\right)$. From Lemma \ref{lem:Gauss1}, the pushforward measure of $\tP_{12}$ by the random vector  $(\bM_1,\bM_2)$ is the $p$-dimensional Gaussian distribution with parameters $(\tbmu,\tbSigma)$. Consequently, $\bM_{12}$ is a Gaussian random vector with mean 
\[
\esp(\bM_{12})=\bA\tbmu
\]
and variance
\[
\var(\bM_{12})=\bA \tbSigma \bA^T.
\]
 
\end{proof}

\subsection{Marginalization and vacuous extension}
\label{subsec:marginGRFV}
In this section, we consider the marginalization and vacuous extension (defined in Section \ref{subsec:margin}) of a GRFV. We assume that variable $\btheta$ taking values in $\reels^p$ is decomposed as $\btheta=(\btheta_1,\btheta_2)$ with $\btheta_1\in \Theta_1=\reels^{p-k}$ and $\btheta_2\in \Theta_2=\reels^k$ for $0<k<p$. 

\paragraph{Marginalization} We start with the following lemma.

\begin{Lem}
\label{lem:proj}
Let $\tF=\GFV(\bm,\bH)$ be a $p$-dimensional Gaussian fuzzy vector with mode $\bm=(\bm_1,\bm_2)$, where $\bm_1\in\Theta_1=\reels^{p-k}$ and $\bm_2\in \Theta_2=\reels^k$ for $0<k<p$, and precision matrix $\bH$ with block decomposition
\[
\bH=\block{\bH_{11}}{\bH_{12}}{\bH_{21}}{\bH_{22}}.
\]
Assume that $\bH_{22}$ is nonsingular. The projection of $\tF$ on  $\Theta_1$, denoted as $\tF\downarrow\Theta_1$ is the Gaussian fuzzy vector $\GFV(\bm_1,\bH'_{11})$ with
\[
\bH'_{11}=\bH_{11} - \bH_{12} \bH_{22}^{-1}\bH_{21}.
\]
\end{Lem}
\begin{proof}
See \ref{app:proj}
\end{proof}

Let us now consider a $p$-dimensional GRFV $\tX\sim\tN(\bmu,\bSigma,\bH)$ representing partial knowledge about  $\btheta=(\btheta_1,\btheta_2)$. The marginal RFS for $\btheta_1$ is given by the following proposition, which follows directly from Lemma \ref{lem:proj}.

\begin{Prop} 
Let $\tX\sim \tN(\bmu,\bSigma,\bH)$ by a $p$-dimensional GRFV taking values in $2^\Theta$, with $\Theta=\Theta_1\times\Theta_2$, where $\Theta_1=\reels^{p-k}$ and $\Theta_2=\reels^k$ for $0<k<p$. Let $\bmu=(\bmu_1,\bmu_2)$ with $\bmu_1\in\Theta_1$ and $\bmu_2\in \Theta_2$, and consider the  block decompositions
\[
\bSigma=\block{\bSigma_{11}}{\bSigma_{12}}{\bSigma_{21}}{\bSigma_{22}} \quad \text{and} \quad \bH=\block{\bH_{11}}{\bH_{12}}{\bH_{21}}{\bH_{22}}.
\]
Assume that $\bH_{22}$ is nonsingular. The marginal  of $\tX$ on $\Theta_1$ is the GRFV $\tX_1\sim\tN(\bmu_1,\bSigma_{11},\bH'_{11})$ with
\[
\bH'_{11}=\bH_{11} - \bH_{12} \bH_{22}^{-1}\bH_{21}.
\]
\end{Prop}

\paragraph{Vacuous extension} We now consider a Gaussian fuzzy vector  $\GFV(\bm_1,\bH_{11})$ in $\Theta_1=\reels^{p-k}$  for $0<k<p$. Its cylindrical extension in $\Theta=\Theta_1\times\Theta_2$, with $\Theta_2=\reels^k$ has the following membership function
\[
\varphi(\bx)=\exp\left(-\frac12 (\bx_1-\bm_1)^T \bH_{11}(\bx_1-\bm_1)\right),
\]
which can be written as
\[
\varphi(\bx)=\exp\left(-\frac12 (\bx-\bm)^T \bH(\bx-\bm)\right),
\]
where $\bm$ is the $p$-dimensional vector
\[
\bm=\begin{pmatrix} \bm_1 \\ \bzero \end{pmatrix}
\]
and
$\bH$ is the $p\times p$ matrix
\begin{equation}
\label{eq:bH}
\bH=\block{\bH_{11}}{\bzero}{\bzero}{\bzero}.
\end{equation}

Given a GRFV $\tX_1\sim\tN(\bmu_1,\bSigma_{11},\bH_{11})$ taking values in $2^{\Theta_1}$, it follows immediately that its vacuous extension   in $\Theta=\Theta_1\times\Theta_2$ is the GRFV
\[
\tX_{1\uparrow (1,2)} \sim \tN(\bmu,\bSigma,\bH)
\]
with
\[
\bmu=\begin{pmatrix} \bmu_1 \\ \bzero \end{pmatrix}, \quad \bSigma=\block{\bSigma_{11}}{\bzero}{\bzero}{\bI_k},
\]
where $\bI_k$ is the $k\times k$ identity matrix, and $\bH$ given by \eqref{eq:bH}.

\paragraph{Noninteractivity} In Section \ref{subsec:margin}, we defined the notion of noninteractive random fuzzy vector. The following proposition gives  a necessary and sufficient condition for a GRFV to be noninteractive.

\begin{Prop}
A $p$-dimensional GRFV $\tX\sim\tN(\bmu,\bSigma,\bH)$ is noninteractive iff matrices $\bSigma$ and $\bH$ are diagonal.
\end{Prop}
\begin{proof}
Let $\tX_1,\ldots,\tX_p$ be the marginals of $\tX$ on each of the $p$ coordinates. Let $\sigma^2_1,\ldots,\sigma^2_p$ and $h_1,\ldots,h_p$ be the diagonal elements of, respectively, $\bSigma$ and $\bH$.  Let $\Omega$ be the set of departure of $\tX$. Let  $\tX_{i\uparrow(1:p)}$ denote the vacuous extension of $\tX_i$ in $\reels^p$, defined by
\[
\tX_{i\uparrow(1:p)}(\omega)(\bx)=\exp\left(-\frac{h}2 (x_i-M_i(\omega))^2\right)
\]
with $M_i\sim N(\mu_i,\sigma^2_i)$. The orthogonal sum
\[
\tX'=\tX_{1\uparrow(1:p)}\oplus\ldots\oplus\tX_{p\uparrow(1:p)}
\]
is given by
\[
\tX'(\omega)(\bx)=\prod_{i=1}^p \exp\left(-\frac{h}2 (x_i-M_i(\omega))^2\right)=\exp\left(-\frac12(\bx-\bM'(\omega))^T\bH'(\bx-\bM'(\omega))\right),
\]
where $\bH'$ is the diagonal matrix with diagonal elements $h_1,\ldots,h_p$, and $\bM'$ is a random vector with mean $\bmu$ and diagonal covariance matrix $\bSigma'$ with diagonal elements $\sigma^2_1,\ldots,\sigma^2_p$.
We have $\tX=\tX'$ iff $\bH=\bH'$ and $\bSigma=\bSigma'$, i.e., if both $\bH$ and $\bSigma$ are diagonal.  
 \end{proof}

\new{
\subsection{Comparison with Dempster's normal belief functions}
\label{subsec:dempster}

In  \cite{dempster01}, Dempster introduced another class  of continuous belief functions in $\reels^p$, called \emph{normal belief functions}\footnote{Ref. \cite{dempster01} was actually available as a working paper from the Statistical Department of Harvard University since 1990, but it only appeared as a book chapter in 2001.}. It is interesting to compare Dempster's model with ours, as  both models generalize the multivariate Gaussian distribution. A normal belief function $Bel$ on $\reels^p$ as defined in \cite{dempster01} is specified by the following components:
\bi
\item An $n$-dimensional subspace $\calS$ of $\reels^p$;
\item A $q$-dimensional partition $\Pi$ of $\calS$ into parallel $n-q$ dimensional subspaces; (If $q=0$, $\Pi=\{\calS\}$);
\item A full-rank $q$-dimensional Gaussian distribution $N(\mu,\Sigma)$ on $\Pi$ if $q>0$, or  the discrete probability measure with mass function  $m(\calS)=1$ if $q=0$. 
\ei 
Belief function $Bel$ is then induced by a random set from $\Pi$, equipped with the normal distribution $N(\mu,\Sigma)$ if $q>0$ or probability mass function $m$ if $q=0$, to the corresponding family of parallel $n-q$ dimensional subspaces of $\calS$. The following special cases are of interest:
\be
\item  If $p=n=q$, $Bel$ is a Gaussian probability distribution on $\reels^p$; 
\item If $p>n=q$, $Bel$ is a Gaussian probability distribution limited to an $n$-dimensional subspace of $\reels^p$;
\item If $p=n$ and $q=0$, $Bel$ is vacuous;
\item If $q=0$ while $p>n>0$,  $Bel$ is logical with $\calS$ as its only focal set; it is then equivalent to specifying $p-n$ linear equations;
\item If $n=q=0$, the true point in $\reels^p$ is known with certainty.
\ee
Like GRFV's, Dempster's normal belief functions thus include the vacuous belief function, Gaussian probability distributions, as well as vacuous extensions of marginal Gaussian distributions. However, the two models are clearly distinct. Dempster's model is based on the combination of  Gaussian probability distributions and linear equations, and is specially useful in relation with linear statistical models such as the Kalman filter \cite{dempster01} or linear regression \cite{monney03}. In contrast, in the GRFV model, focal sets are fuzzy subsets of $\reels^n$ ($n\le p$) with Gaussian membership functions, or cylindrical extensions of such fuzzy subsets. This model allows us to represent not only probabilistic and logical evidence, but also fuzzy information. In particular, it includes Gaussian probability distribution and Gaussian possibility distributions as special cases. We could attempt  to design an even more general model that would contain both Dempster's normal belief functions and belief functions induced by GRFV's as special cases.  Such a model would allow us to reason with Gaussian probability and possibility distributions as well as with linear equations. The rigorous development of such a model is left for further research.
}

\section{Conclusions}
\label{sec:concl}

In this paper, continuing a study started in \cite{denoeux21a} with the finite case, we have introduced a theory of epistemic random fuzzy sets in a general setting. An epistemic random fuzzy set represents a piece of evidence, which may be crisp or fuzzy. This framework generalizes both epistemic random sets as considered in the Dempster-Shafer theory of belief functions, and possibility distributions considered in possibility theory. Independent epistemic random fuzzy sets are combined by the generalized product-intersection rule, which extends both Dempster's rule for combining belief functions and the product intersection rule for combining possibility distributions. 

In addition, we have also introduced Gaussian random fuzzy numbers (GRFN's) and their multidimensional extensions, Gaussian random fuzzy vectors (GRFV's) as practical models of random fuzzy subsets of, respectively,  $\reels$ and $\reels^p$ with $p\ge 2$. A GRFN is described by three parameters:  its mode $m$, its variance $\sigma^2$ and its precision $h$. In this setting, a Gaussian random variable can be seen as an infinitely precise GRFN ($h=+\infty$),  while a Gaussian possibility distribution is a  constant GRFN ($\sigma^2=0$). A maximally imprecise GRFN such that  $h=0$ is said to be vacuous: it represents complete ignorance. In GRFV's, the mode becomes a $p$-dimensional vector, while the variance and precision become positive semi-definite $p\times p$ square matrices. The practical convenience of GRFN's and GRFV's arises from the fact that they can easily be combined by the generalized product-intersection rule. Also, formulas for the projection and marginal extension fo GRFV's have been derived.

This work opens up several perspectives. Using random fuzzy sets and, in particular, GRFN's to represent expert knowledge about numerical quantities will require the development of adequate elicitation procedures. We also consider using this framework in machine learning, to quantify prediction uncertainty in regression problems. \new{Finally, the extension of the model introduced in this paper to take into account linear equations, as well as the development of computational procedures for reasoning with GRFV's over many variables are promising avenues for further  research.}

\section*{References}

\appendix

\section{Proof of Proposition \ref{prop:dempster_assoc} }
\label{app:dempster_assoc}

Commutativity is obvious. To prove associativity, let us consider three random sets $(\Omega_i,\sigma_i,P_i,\Theta,\sigma_\Theta,\bbX_i)$, $i=1,2,3$. Consider the combined random set
\begin{equation}
\label{eq:comb1}
(\Omega_1\times\Omega_2\times\Omega_3,\sigma_1\otimes\sigma_2\otimes\sigma_3,P_{123},\Theta, \sigma_\Theta,\bbX_{1\cap 2\cap 3}),
\end{equation}
where
\[
\bbX_{1\cap 2\cap 3}(\omega_1,\omega_2,\omega_3)=\bbX_1(\omega_1) \cap  \bbX_2(\omega_2) \cap \bbX_3(\omega_3),
\]
\[
P_{123}=(P_1\times P_2 \times P_3)(\cdot \mid \Theta^*_{123}),
\]
and
\[
\Theta^*_{123}=\{(\omega_1,\omega_2,\omega_3)\in \Omega_1\times \Omega_2 \times \Omega_2: \bbX_{1\cap 2\cap 3}(\omega_1,\omega_2,\omega_3)\neq \emptyset\}.
\]
We will show that we get the same result by combining $\bbX_1$ with $\bbX_2$ first, and then combining the result with $\bbX_3$. Combining the first two random sets, we get
\[
(\Omega_1\times\Omega_2,\sigma_1\otimes\sigma_2,P_{12},\Theta, \sigma_\Theta,\bbX_{1\cap 2}),
\]
with $\bbX_{1\cap 2}(\omega_1,\omega_2)=\bbX_1(\omega_1)\cap \bbX_2(\omega_2)$, $P_{12}=(P_1\times P_2)(\cdot\mid \Theta^*_{12})$ and 
\[
\Theta^*_{12}=\{(\omega_1,\omega_2)\in \Omega_1\times \Omega_2 : \bbX_{1\cap 2}(\omega_1,\omega_2)\neq \emptyset\}.
\] 
Combining it with $\bbX_3$ we get
\begin{equation}
\label{eq:comb2}
(\Omega_1\times\Omega_2\times\Omega_3,\sigma_1\otimes\sigma_2\otimes\sigma_3,P_{(12)3},\Theta, \sigma_\Theta,\bbX_{1\cap 2\cap 3}),
\end{equation}
with $P_{(12)3}=(P_{12}\times P_3)(\cdot\mid \Theta^*_{123})$. Comparing \eqref{eq:comb1} and \eqref{eq:comb2}, we see that we only need to show that $P_{(12)3}=P_{123}$. For any event $C\subseteq \Theta^*_{123}$ and any $\omega_3\in\Omega_3$, let $C_{\omega_3}=\{(\omega_1,\omega_2)\in \Omega_1\times\Omega_2: (\omega_1,\omega_2,\omega_3)\in C\}$. By definition of the product measure $P_{12}\times P_3$ (see \cite[page 144] {halmos50}), we have
\begin{equation}
\label{eq:P12_3}
P_{(12)3}(C)=\frac{(P_{12}\times P_3)(C)}{(P_{12}\times P_3)(\Theta^*_{123})}=\frac{1}{(P_{12}\times P_3)(\Theta^*_{123})}\int P_{12}(C_{\omega_3})dP_3(\omega_3) 
\end{equation}
Now, as $C\subseteq \Theta^*_{123}$, for any $(\omega_1,\omega_2)\in C_{\omega_3}$, $\bbX_1(\omega_1)\cap\bbX_2(\omega_2) \neq\emptyset$. Consequently, $C_{\omega_3}\subseteq \Theta^*_{12}$, so
\begin{equation}
\label{eq:P12bis}
P_{12}(C_{\omega_3})=\frac{(P_1\times P_2)(C_{\omega_3})}{(P_1\times P_2)(\Theta^*_{12})}.
\end{equation}
From \eqref{eq:P12_3} and \eqref{eq:P12bis}, we get
\begin{subequations}
\begin{align}
P_{(12)3}(C)&=\frac{1}{(P_{12}\times P_3)(\Theta^*_{123}) (P_1\times P_2)(\Theta^*_{12})}\int (P_1\times P_2)(C_{\omega_3})dP_3(\omega_3)\\
\label{eq:P12_3bis}
 &=\frac{(P_1\times P_2\times P_3)(C)}{(P_{12}\times P_3)(\Theta^*_{123}) (P_1\times P_2)(\Theta^*_{12})}.
\end{align}
\end{subequations}
Now,
\begin{equation}
\label{eq:P123bis}
P_{123}(C)=\frac{(P_1\times P_2\times P_3)(C)}{(P_{1}\times P_2\times P_3)(\Theta^*_{123})}.
\end{equation}
As $P_{(12)3}(\Theta^*_{123})=P_{123}(\Theta^*_{123})=1$, the denominators in \eqref{eq:P12_3bis} and \eqref{eq:P123bis} are equal, and $P_{(12)3}=P_{123}$.

\section{Proof of Proposition \ref{prop:assoc_soft}}
\label{app:assoc_soft}
Commutativity is obvious. To prove associativity, consider three random fuzzy sets
\[
(\Omega_i,\sigma_{i},P_i,\Theta, \sigma_\Theta,\tX_i),  \quad i=1,2, 3.
\] 
Let $\tTheta^*_{12}$ be the fuzzy subset of $\Omega_1\times\Omega_2$ with membership function
\[
\tTheta_{12}^*(\omega_1,\omega_2)= \height\left(\tX_1(\omega_1) \tX_2(\omega_2)\right),
\]
and let $\tTheta^*_{(12)3}$ and $\tTheta^*_{123}$ be the fuzzy subsets of $\Omega_1\times\Omega_2\times\Omega_3$ defined, respectively, as
\[
\tTheta_{(12)3}^*(\omega_1,\omega_2,\omega_3)= \height\left(\left[\tX_1(\omega_1)\varodot\tX_2(\omega_1)\right]  \tX_3(\omega_3)\right)
\]
and
\[
\tTheta_{123}^*(\omega_1,\omega_2,\omega_3)= \height\left(\tX_1(\omega_1) \tX_2(\omega_2) \tX_3(\omega_3)\right).
\]
Let $\tP_{12}=(P_1\times P_2)(\cdot \mid \tTheta_{12}^*)$, $\tP_{(12)3}=(\tP_{12}\times P_3)(\cdot \mid \tTheta_{(12)3}^*)$, and $\tP_{123}=(P_1\times P_2 \times P_3)(\cdot \mid \tTheta_{123}^*)$. We only need to show that $\tP_{(12)3}=\tP_{123}$. For any $B\in \sigma_1\otimes\sigma_2\otimes\sigma_3$, we have
\begin{align*}
\tP_{(12)3}(B) &\propto \int_{\Omega_1\times\Omega_2}\int_{\Omega_3}B(\omega_1,\omega_2,\omega_3) \height\left(\left[\tX_1(\omega_1)\varodot\tX_2(\omega_1)\right]  \tX_3(\omega_3)\right) dP_3(\omega_3) d\tP_{12}(\omega_1,\omega_2) \\
&\propto \int_{\Omega_1}\int_{\Omega_2}\int_{\Omega_3}B(\omega_1,\omega_2,\omega_3) \height\left(\left[\tX_1(\omega_1)\varodot\tX_2(\omega_1)\right]  \tX_3(\omega_3)\right) \times \\
& \hspace{5cm} \height\left(\tX_1(\omega_1) \tX_2(\omega_2)\right)dP_3(\omega_3) dP_{2}(\omega_2) dP_1(\omega_1).
\end{align*}
Now,
\begin{align*}
\height\left(\left[\tX_1(\omega_1)\varodot\tX_2(\omega_1)\right]  \tX_3(\omega_3)\right)&=
\height\left(\frac{\tX_1(\omega_1)\tX_2(\omega_1)}{\height(\tX_1(\omega_1)\tX_2(\omega_2))} \tX_3(\omega_3)\right)\\
&=\frac{\height(\tX_1(\omega_1)\tX_2(\omega_2)\tX_3(\omega_3))}{\height(\tX_1(\omega_1)\tX_2(\omega_1))}.
\end{align*}
Hence,
\[
\tP_{(12)3}(B) \propto \int_{\Omega_1}\int_{\Omega_2}\int_{\Omega_3}B(\omega_1,\omega_2,\omega_3)  \height\left(\tX_1(\omega_1) \tX_2(\omega_2)  \tX_3(\omega_3)\right)dP_3(\omega_3) dP_{2}(\omega_2) dP_1(\omega_1),
\]
which proves that $\tP_{(12)3}=\tP_{123}$, and the associativity of $\oplus$.

\section{Proof of Proposition \ref{prop:contourRFS}}
\label{app:contourRFS}

We have
\begin{align}
pl_\tX(x)=&\esp_M[\varphi(x;M,h)]\\
=&\int_{-\infty}^{+\infty} \varphi(x;m,h) \phi(m;\mu,\sigma) dm\\
=&\frac{1}{\sigma\sqrt{2\pi}}\int_{-\infty}^{+\infty} \exp\left(-\frac{h}{2}(x-m)^2\right) \exp\left( -\frac{(m-\mu)^2}{2\sigma^2}\right)dm.
\end{align}
From Proposition \ref{prop:prodphi}, the integrand can be written as
\[
\exp\left(-\frac{(m-\mu_0)^2}{2\sigma_0^2}\right) \exp\left(-\frac{h(x-\mu)^2}{2(1+h\sigma^2)}\right),
\]
with
\[
\mu_0=\frac{xh+\mu/\sigma^2}{h+1/\sigma^2}=\frac{xh\sigma^2+\mu}{h\sigma^2+1}
\]
and
\[
\sigma_0=\sqrt{\frac{1}{h+1/\sigma^2}}=\frac{\sigma}{\sqrt{1+h\sigma^2}}.
\]
Consequently,
\begin{align}
pl_\tX(x)=&\frac{1}{\sigma\sqrt{2\pi}} \exp\left(-\frac{h(x-\mu)^2}{2(1+h\sigma^2)}\right) \underbrace{\int_{-\infty}^{+\infty} \exp\left(-\frac{(m-\mu_0)^2}{2\sigma_0^2}\right)dm}_{\sigma_0\sqrt{2\pi}}\\
&= \frac{1}{\sqrt{1+h\sigma^2}} \exp\left(-\frac{h(x-\mu)^2}{2(1+h\sigma^2)}\right).
\end{align}

\section{Proof of Proposition \ref{prop:BelPlRFS}}
\label{app:BelPlRFS}

If $h=0$, we have, trivially, $Bel_\tX([x,y])=0$ and $Pl_\tX([x,y])=1$ for all $x\le y$. Let us assume that $h>0$.  We have
\begin{multline}
Pl_\tX([x,y])=\Pr(M\le x) \esp[\varphi(x;M,h)\mid M\le x]+\\
\Pr(x < M\le y) \times 1+ 
\Pr(M> y) \esp[\varphi(y;M,h)\mid M>y],
\end{multline}
which can be written as
\begin{multline}
Pl_\tX([x,y])=\Phi\fracpar{x-\mu}{\sigma} \esp[\varphi(x;M,h)\mid M\le x]+\\
 \Phi\fracpar{y-\mu}{\sigma} - \Phi\fracpar{x-\mu}{\sigma}+ \\
\left[1-\Phi\fracpar{y-\mu}{\sigma}\right] \esp[\varphi(y;M,h)\mid M>y].
\end{multline}
Conditionally on $M\le x$, $M$ has a truncated normal distribution on $(-\infty,x]$ with pdf
\[
f(m)=\frac{1}{\sigma\sqrt{2\pi}}\frac{\exp\fracpar{-(m-\mu)^2}{2\sigma^2}}{\Phi\fracpar{x-\mu}{\sigma}} \one_{(-\infty,x]}(m).
\]
Consequently,
\begin{equation}
\label{eq:Ephi}
\esp[\varphi(x;M,h)\mid M\le x]=
\frac{1}{\sigma\sqrt{2\pi}}\frac{1}{\Phi\fracpar{x-\mu}{\sigma}} \underbrace{\int_{-\infty}^x \exp\left(-\frac{h}{2}(x-m)^2\right) \exp\left( -\frac{(m-\mu)^2}{2\sigma^2}\right)dm}_{I}.
\end{equation}
From Proposition \ref{prop:prodphi},  integral $I$ in \eqref{eq:Ephi} can be written as
\[
I = \sigma_0\sqrt{2\pi}\Phi\fracpar{x-\mu_0}{\sigma_0}\exp\left(- \frac{(x-\mu)^2}{2(h^{-1}+\sigma^2)}\right),
\]
with 
\[
\mu_0=\frac{xh\sigma^2+\mu}{h\sigma^2+1} \quad \text{and} \quad \sigma_0=\frac{\sigma}{\sqrt{h\sigma^2+1}}.
\]
Consequently,
\[
\esp[\varphi(x;M,h)\mid M\le x]=
\frac{1}{\Phi\fracpar{x-\mu}{\sigma}} pl_\tX(x) \Phi\fracpar{x-\mu}{\sigma\sqrt{h\sigma^2+1}}.
\]
Using similar calculations, we find
\[
\esp[\varphi(y;M,h)\mid M>y]=
\frac{1}{1-\Phi\fracpar{y-\mu}{\sigma}} pl_\tX(y) \left[1-\Phi\fracpar{y-\mu}{\sigma\sqrt{h\sigma^2+1}}\right],
\]
which concludes the proof of \eqref{eq:plint}. 

Now, let us consider \eqref{eq:belint}. We have
\[
Bel_\tX([x,y])=1-Pl_\tX((-\infty,x] \cup [y,+\infty)),
\]
and
\begin{multline}
Pl_\tX((-\infty,x] \cup [y,+\infty))=\Pr(M\le x) \times 1 +\\
\Pr(x < M\le (x+y)/2)\esp[\varphi(x;M,h)\mid x < M\le (x+y)/2]+\\
\Pr((x+y)/2 < M\le y)\esp[\varphi(y;M,h)\mid (x+y)/2 < M\le y]+
\Pr(M> y) \times 1,
\end{multline}
which can be written as
\begin{multline}
Pl_\tX((-\infty,x] \cup [y,+\infty))=\Phi\fracpar{x-\mu}{\sigma}  +\\
\left[\Phi\fracpar{(x+y)/2-\mu}{\sigma} - \Phi\fracpar{x-\mu}{\sigma}\right]\esp[\varphi(x;M,h)\mid x < M\le (x+y)/2]+\\
\left[\Phi\fracpar{y-\mu}{\sigma}-\Phi\fracpar{(x+y)/2-\mu}{\sigma}\right]\esp[\varphi(y;M,h)\mid (x+y)/2 < M\le y]+\\
1-\Phi\fracpar{y-\mu}{\sigma}.
\end{multline}
Conditionally on $x < M\le (x+y)/2$, $M$ has a truncated normal distribution on $(x,(x+y)/2]$ with pdf
\[
f(m)=\frac{1}{\sigma\sqrt{2\pi}}\frac{\exp\fracpar{-(m-\mu)^2}{2\sigma^2}}{\Phi\fracpar{(x+y)/2-\mu}{\sigma} - \Phi\fracpar{x-\mu}{\sigma}}  \one_{(x,(x+y)/2]}(m).
\]
Consequently,
\begin{multline}
\label{eq:Ephi2}
\esp[\varphi(x;M,h)\mid x < M\le (x+y)/2]=
\frac{1}{\sigma\sqrt{2\pi}}\frac{1}{\Phi\fracpar{(x+y)/2-\mu}{\sigma} - \Phi\fracpar{x-\mu}{\sigma}} \times \\\underbrace{\int_{x}^{(x+y)/2} \exp\left(-\frac{h}{2}(x-m)^2\right) \exp\left( -\frac{(m-\mu)^2}{2\sigma^2}\right)dm}_{I'}.
\end{multline}
The integral in \eqref{eq:Ephi2} is, with the same notations as before,
\[
I'=\sigma_0\sqrt{2\pi}\left[\Phi\fracpar{(x+y)/2-\mu_0}{\sigma_0} - \Phi\fracpar{x-\mu_0}{\sigma_0}\right]\exp\left(- \frac{(x-\mu)^2}{2(h^{-1}+\sigma^2)}\right).
\]
Consequently,
\begin{multline}
\esp[\varphi(x;M,h)\mid x < M\le (x+y)/2]=\\
\frac{1}{\Phi\fracpar{(x+y)/2-\mu}{\sigma} - \Phi\fracpar{x-\mu}{\sigma}} pl_\tX(x) \left[\Phi\fracpar{(x+y)/2-\mu+ h\sigma^2(y-x)/2}{\sigma\sqrt{h\sigma^2+1}} \right.\\
\left. - \Phi\fracpar{x-\mu}{\sigma\sqrt{h\sigma^2+1}}\right].
\end{multline}
Similarly, we find
\begin{multline}
\esp[\varphi(y;M,h)\mid (x+y)/2 < M\le y]=\\
\frac{1}{\Phi\fracpar{y-\mu}{\sigma} - \Phi\fracpar{(x+y)/2-\mu}{\sigma}} pl_\tX(y) \left[\Phi\fracpar{y-\mu}{\sigma\sqrt{h\sigma^2+1}} - \right.\\
\left. \Phi\fracpar{(x+y)/2-\mu-(y-x)h\sigma^2/2}{\sigma\sqrt{h\sigma^2+1}}\right].
\end{multline}
The expressions of $Pl_\tX((-\infty,x] \cup [y,+\infty))$ and $Bel_\tX([x,y])$ follow.

\section{Proof of Proposition \ref{prop:expect}}
\label{app:expect}

Let $\tX(\omega)=\GFN(M(\omega),h)$ be the image of $\omega\in\Omega$ by $\tX$, with $M\sim N(\mu,\sigma^2)$. For any $\alpha\in (0,1]$, its alpha-cut is the random interval
\[
\cut{\alpha}{\tX}(\omega)=\left[M(\omega)-\sqrt{\frac{-2\ln\alpha}h},M(\omega)+\sqrt{\frac{-2\ln\alpha}h}\right].
\]
Consequently, from \eqref{eq:calclowerupper}, the lower and upper expectation of $\tX$ are
\[
\esp_*(\tX)=\mu-\int_0^1 \sqrt{\frac{-2\ln\alpha}h} d\alpha,
\]
and
\[
\esp^*(\tX)=\mu+\int_0^1 \sqrt{\frac{-2\ln\alpha}h} d\alpha.
\]
By the change of variable $\beta=\sqrt{-2(\ln\alpha)/h}$, we get
\[
\int_0^1 \sqrt{\frac{-2\ln\alpha}h} d\alpha=h \int_0^{+\infty} \beta^2\exp\left(-\frac{h\beta^2}2\right) d\beta. 
\]
Now, the second-order moment of the normal distribution $N(0,1/h)$ is
\[
\sqrt{\frac{h}{2\pi}}\int_{-\infty}^{+\infty} \beta^2\exp\left(-\frac{h\beta^2}2\right) d\beta=\frac{1}h,
\]
from which we get
\[
h \int_0^{+\infty} \beta^2\exp\left(-\frac{h\beta^2}2\right) d\beta=h \cdot \frac1h \sqrt{\frac{\pi}{2h}}=\sqrt{\frac{\pi}{2h}}.
\]

\section{Proof of Lemma \ref{lem:Gauss}}
\label{app:Gauss} 

The conditional density of $(M_1,M_2)$ is
\begin{equation}
\label{eq:condf}
f(m_1,m_2\mid \tF) = \frac{f(m_1,m_2) \tF(m_1,m_2)}{\iint f(m_1,m_2) \tF(m_1,m_2) dm_1dm_2}.
\end{equation}
The numerator on the right-hand side of \eqref{eq:condf} is
\begin{multline}
\label{eq:condf1}
\frac{1}{2\pi\sigma_1\sigma_2} \exp\left\{-\frac12 \left[\left(\frac{m_1-\mu_1}{\sigma_1}\right)^2 + \left(\frac{m_2-\mu_2}{\sigma_2}\right)^2\right]\right\}  \exp\left\{-\frac{\oh(m_1-m_2)^2}{2}\right\}\\
= \frac{1}{2\pi\sigma_1\sigma_2} \exp\left\{-\frac12 \left[m_1^2\left(\frac1{\sigma_1^2}+\oh\right) - \frac{2m_1\mu_1}{\sigma_1^2} +\frac{\mu_1^2}{\sigma_1^2} +  \right.\right.\\
\left.\left. m_2^2\left(\frac1{\sigma_2^2}+\oh\right) -\frac{2m_2\mu_2}{\sigma_2^2} +\frac{\mu_2^2}{\sigma_2^2} -2\oh m_1m_2\right]\right\}.
\end{multline}
Now, the two-dimensional Gaussian density with parameters $(\tmu_1,\tmu_2,\tsigma_1,\tsigma_2,\rho)$ equals
\begin{multline}
\label{eq:condf2}
\frac{1}{2\pi\tsigma_1\tsigma_2\sqrt{1-\rho^2}} \exp\left\{-\frac1{2(1-\rho)^2} \left[ \left(\frac{m_1-\tmu_1}{\tsigma_1}\right)^2 - \right.\right.\\
\left.\left.2\rho\left(\frac{m_1-\tmu_1}{\tsigma_1}\right)\left(\frac{m_2-\tmu_2}{\tsigma_2}\right) + \left(\frac{m_2-\tmu_2}{\tsigma_2}\right)^2 \right]\right\}.
\end{multline}
Equating the second and first-order terms inside the exponentials in  \eqref{eq:condf1} and \eqref{eq:condf2} gives us
\begin{subequations}
\begin{align}
\label{eq:parGsig1bis}
\tsigma_1&=\frac1{1-\rho^2}\left(\frac1{\sigma_1^2} + \oh\right)^{-1}\\
\label{eq:parGsig2bis}
\tsigma_2&=\frac1{1-\rho^2}\left(\frac1{\sigma_2^2} + \oh\right)^{-1}\\
\label{eq:parGrhobis}
\rho &= \frac{\oh\sigma_1\sigma_2}{\sqrt{(1+\oh\sigma_1^2)(1+\oh\sigma_2^2)}}\\
\label{eq:parGmu1bis}
\tmu_1 &= \frac{\mu_1\tsigma_1^2}{\sigma_1^2} + \rho \mu_2 \frac{\tsigma_1\tsigma_2}{\sigma_2^2}\\
\label{eq:parGmu2bis}
\tmu_2 &= \frac{\mu_2\tsigma_2^2}{\sigma_2^2} + \rho \mu_1 \frac{\tsigma_1\tsigma_2}{\sigma_1^2}.
\end{align}
\end{subequations} 
Replacing  $\rho$ by its expression \eqref{eq:parGrhobis} in \eqref{eq:parGsig1bis} and \eqref{eq:parGsig2bis} yields \eqref{eq:parGsig1} and \eqref{eq:parGsig2}. Replacing $\rho$, $\tsigma_1$ and $\tsigma_2$ by their expressions in \eqref{eq:parGmu1bis} and \eqref{eq:parGmu2bis} gives \eqref{eq:parGmu1} and \eqref{eq:parGmu2}.

Finally, the degree of conflict  between GRFN's $\tX_1\sim \tN(\mu_1,\sigma_1^2,h_1)$ and $\tX_2\sim \tN(\mu_2,\sigma_2^2,h_2)$ is
\[
\kappa=1-(P_1\times P_2)(\tTheta^*),
\]
with 
\[
(P_1\times P_2)(\tTheta^*) = \iint f(m_1,m_2) \tF(m_1,m_2) dm_1dm_2.
\]
Taking the ratio of \eqref{eq:condf1} to \eqref{eq:condf2}, we get
\begin{multline*}
 \iint f(m_1,m_2) \tF(m_1,m_2) dm_1dm_2=\\
 \frac{\tsigma_1\tsigma_2}{\sigma_1\sigma_2}\sqrt{1-\rho^2}\exp\left\{-\frac12 \left[\frac{\mu_1^2}{\sigma_1^2}+ \frac{\mu_2^2}{\sigma_2^2}\right] +\frac1{2(1-\rho^2)} \left[\frac{\tmu_1^2}{\tsigma_1^2}+ \frac{\tmu_2^2}{\tsigma_2^2} - 2\rho \frac{\tmu_1\tmu_2}{\tsigma_1\tsigma_2}\right]\right\}.
\end{multline*}

\section{Proof of Proposition \ref{prop:combRV}}
\label{app:combRV}

From \eqref{eq:h12}, $h_{12}=+\infty$ and the combined GRFN $\tN(\tmu_{12},\tsigma_{12}^2,h_{12})$ is probabilistic. From \eqref{eq:mu12} and \eqref{eq:sig12}, 
\[
\tmu_{12}=\lim_{h_1\rightarrow +\infty} \frac{\tmu_1+\frac{h_2}{h_1} \tmu_2}{1+\frac{h_2}{h_1}}=\tmu_1,
\]
and
\[
\tsigma_{12}^2=\lim_{h_1\rightarrow +\infty} \frac{ \tsigma^2_1+\frac{h^2_2}{h^2_1} \tsigma^2_2 + 2 \rho \frac{h_2}{h_1} \tsigma_1\tsigma_2}{(1+\frac{h_2}{h_1})^2}=\tsigma^2_1.
\]
From \eqref{eq:hbar},
\[
\oh=\lim_{h_1\rightarrow +\infty} \frac{h_2}{1+\frac{h_2}{h_1}}=h_2.
\]
From \eqref{eq:parGmu1} and \eqref{eq:parGsig1},
\[
\tmu_1 = \frac{\mu_1(1+h_2\sigma_2^2)+  \mu_2 h_2 \sigma_1^2}{1+h_2 (\sigma_1^2+\sigma_2^2)},
\]
and
\[
\tsigma_1^2= \frac{\sigma_1^2(1+h_2 \sigma_2^2)}{1+h_2 (\sigma_1^2+\sigma_2^2)}.
\]
Now, using Proposition \ref{prop:prodphi}, the product of the probability density of $X_1$ and the contour function of $\tX_2$ can be written as
\begin{align*}
f_{X_1}(x)pl_{\tX_2}(x) &\propto \exp\left(-\frac12 \frac{(x-\mu_1)^2}{\sigma_1^2}\right)\exp\left(- \frac{h_2(x-\mu_2)^2}{2(1+h_2\sigma_2^2)}\right)\\
&\propto \exp\left(-\frac{1}{2\sigma_{12}^2} (x-\mu_{12})^2\right),
\end{align*}
with
\[
\frac1{\sigma_{12}^2}=\frac1{\sigma_1^2}+ \frac{h_2}{1+h_2\sigma_2^2}=\frac{1+h_2 (\sigma_1^2+\sigma_2^2)}{\sigma_1^2(1+h_2 \sigma_2^2)}
\]
and
\[
\mu_{12}=\frac{\frac1{\sigma_1^2}\mu_1+\frac{h_2}{1+h_2\sigma_2^2}\mu_2}{\frac1{\sigma_1^2}+ \frac{h_2}{1+h_2\sigma_2^2}}=\frac{\mu_1(1+h_2\sigma_2^2)+  \mu_2 h_2 \sigma_1^2}{1+h_2 (\sigma_1^2+\sigma_2^2)}.
\]
We can check that $\mu_{12}=\tmu_1$ and $\sigma_{12}^2=\tsigma_{1}^2$.

\section{Proof of Proposition \ref{prop:contourRFS1}}
\label{app:contourRFS1}

We have
\begin{align}
pl_\tX(\bx)&=\esp_\bM[\varphi(\bx;\bM,\bH)]\\
&=\int_{\reels^p} \varphi(\bx;\bm,\bH) \phi(\bm;\bmu,\Sigma) d\bm\\
&=\frac{1}{(2\pi)^{p/2}|\bSigma|^{1/2}} \int_{\reels^p}  \exp\left(-\frac{1}{2}(\bx-\bm)^T \bH(\bx-\bm)\right) \times \\
& \hspace{4cm}\exp\left( -\frac12(\bm-\bmu)\bSigma^{-1}(\bm-\bmu)\right)d\bm.
\end{align}
From Proposition \ref{prop:prodphi}, the integrand can be written as
\[
\exp\left(-\frac12 (\bm-\bmu_0)^T\bSigma_0^{-1}(\bm-\bmu_0)\right) \exp\left(-\frac12 (\bx-\bmu)^T(\bH^{-1}+\bSigma)^{-1} (\bx-\bmu)\right),
\]
with
\[
\bmu_0=(\bH+\bSigma^{-1})^{-1}(\bH\bx+\bSigma^{-1}\bmu)
\]
and
\[
\bSigma_0=(\bH+\bSigma^{-1})^{-1}.
\]
Consequently,
\begin{align}
pl_\tX(\bx)&=\frac{1}{(2\pi)^{p/2}|\bSigma|^{1/2}} \exp\left(-\frac12 (\bx-\bmu)^T(\bH^{-1}+\bSigma)^{-1} (\bx-\bmu)\right) \times  \\
& \hspace{4cm}\underbrace{\int_{\reels^p}\exp\left(-\frac12 (\bm-\bmu_0)^T\bSigma_0^{-1}(\bm-\bmu0)\right)d\bm}_{(2\pi)^{p/2}|\bSigma_0|^{1/2}}\\
&= \left(\frac{|\bSigma_0|}{|\bSigma|}\right)^{1/2} \exp\left(-\frac12 (\bx-\bmu)^T(\bH^{-1}+\bSigma)^{-1} (\bx-\bmu)\right)\\
&=\frac{1}{|I_p+\bSigma H|^{1/2}}\exp\left(- \frac12 (\bx-\bmu)^T(\bH^{-1}+\bSigma)^{-1}(\bx-\bmu)\right).
\end{align}

\section{Proof of Lemma \ref{lem:Gauss1}}
\label{app:Gauss1} 

The conditional density of $\bM=(\bM_1,\bM_2)$ is
\begin{equation}
\label{eq:condf3}
f(\bm_1,\bm_2\mid \tF) = \frac{f(\bm_1,\bm_2) \tF(\bm_1,\bm_2)}{ \int_{\reels^{2p}}  f(\bm_1,\bm_2) \tF(\bm_1,\bm_2) d\bm_1d\bm_2}.
\end{equation}
The numerator on the right-hand side of \eqref{eq:condf3} is
\begin{multline}
\label{eq:condfmult1}
f(\bm_1,\bm_2) \tF(\bm_1,\bm_2)=\phi(\bm_1;\bmu_1,\bSigma_1)\phi(\bm_2;\bmu_2,\bSigma_2) \times \\
\exp\left\{-\frac12 (\bm_1-\bm_2)^T\obH(\bm_1-\bm_2)\right\},
\end{multline}
which can be written as 
\[
f(\bm_1,\bm_2) \tF(\bm_1,\bm_2)=\frac{1}{(2\pi)^p |\bSigma_1\bSigma_2|^{1/2}} \exp\left(-\frac{Z}2 \right)
\]
with
\begin{multline}
\label{eq:Z}
Z= \bm_1^T(\bSigma_1^{-1}+\obH)\bm_1+\bm_2^T(\bSigma_2^{-1}+\obH)\bm_2 -2 \bm_1^T\obH\bm_2 -2\bm_1^T\bSigma_1^{-1}\bmu_1 - \\
2\bm_2^T\bSigma_2^{-1}\bmu_2 + \bmu_1^T\bSigma_1^{-1}\bmu_1 + \bmu_2^T\bSigma_2^{-1}\bmu_2.
\end{multline}
Now, the $2p$-dimensional Gaussian density with mean $\tbmu$ and covariance matrix $\tbSigma$ equals
\begin{equation}
\label{eq:condfmult2}
\phi(\bm;\tbmu,\tbSigma)=\frac{1}{(2\pi)^p |\tbSigma|^{1/2}} \exp\left\{-\frac12 (\bm-\bmu)^T \tbSigma^{-1} (\bm-\bmu)\right\}.
\end{equation}
Decomposing vector $\tbmu$ as $\tbmu=(\tbmu_1,\tbmu_2)$, with $\tbmu_1,\tbmu_2\in \reels^p$, and $\tbSigma^{-1}$ as
\[
\tbSigma^{-1}=\block{\bA}{\bB}{\bB}{\bC},
\]
where $\bA$, $\bB$ and $\bC$ are $p\times p$ matrices, we can rewrite \eqref{eq:condfmult2} as
\[
\phi(\bm;\tbmu,\tbSigma)=\frac{1}{(2\pi)^p |\tbSigma|^{1/2}} \exp\left\{-\frac12 Z'\right\}
\] 
with
\begin{multline}
\label{eq:Zp}
Z'= \bm_1^T \bA \bm_1-2 \bm_1^T \bA \tmu_1 + \tmu_1^T\bA\tmu_1 +\bm_2^T \bC \bm_2-2 \bm_2^T \bC \tmu_2 + \tmu_2^T\bC\tmu_u+\\
2 \bm_2^T \bB \bm_1 - 2 \bm_2^T \bB \bmu_1- 2 \bm_1^T \bB \bmu_2 + 2 \bmu_2^T \bB \bmu_1.
\end{multline}
Equating the second-order terms in \eqref{eq:Z} and \eqref{eq:Zp}, we get
\[
\bA=\bSigma_1^{-1}+\obH, \quad \bC=\bSigma_2^{-1}+\obH, \quad \bB=-\obH.
\]
Equating the first-order terms, we get
\begin{subequations}
\begin{align}
\label{eq:bmu1}
\bSigma_1^{-1}\bmu_1&=\bA\tbmu_1+\bB\tbmu_2=(\bSigma_1^{-1}+\obH)\tbmu_1-\obH\tbmu_2,\\
\label{eq:bmu2}
\bSigma_2^{-1}\bmu_2&=\bB\tbmu_1+\bC\tbmu_2=-\obH\tbmu_1+ (\bSigma_2^{-1}+\obH)\tbmu_2.
\end{align}
\end{subequations}
Multiplying both sides of \eqref{eq:bmu1} and \eqref{eq:bmu2} by $\obH^{-1}$, we get
\begin{align}
(\obH^{-1}\bSigma_1^{-1}+\bI_p)\tbmu_1-\tbmu_2&= \obH^{-1}\bSigma_1^{-1}\bmu_1\\
-\tbmu_1+ (\obH^{-1}\bSigma_2^{-1}+\bI_p)\tbmu_2&=\obH^{-1}\bSigma_2^{-1}\bmu_2,
\end{align}
which can be written in matrix form
\begin{equation*}
\begin{pmatrix}
\obH^{-1}\bSigma_1^{-1}+\bI_p& -\bI_p\\
 -\bI_p &\obH^{-1}\bSigma_2^{-1}+\bI_p
\end{pmatrix}\begin{pmatrix}
\tbmu_1\\
\tbmu_2
\end{pmatrix}=
\begin{pmatrix}
\obH^{-1}\bSigma_1^{-1}& \bzero\\
 \bzero &\obH^{-1}\bSigma_2^{-1}
\end{pmatrix}
\begin{pmatrix}
\bmu_1\\
\bmu_2
\end{pmatrix},
\end{equation*}
from which we obtain \eqref{eq:tbmu}. 

Finally, the degree of conflict between GRFV's $\tX_1\sim\tN(\bmu_1,\bSigma_1,\bH_1)$ and $\tX_2\sim\tN(\bmu_2,\bSigma_2,\bH_2)$ is 
\[
\kappa=1-(P_1\times P_2)(\tTheta^*) =1- \int_{\reels^{2p}} f(\bm_1,\bm_2) \tF(\bm_1,\bm_2) d\bm_1d\bm_2.
\]
Taking the ratio of \eqref{eq:condfmult1} to \eqref{eq:condfmult2}, we get
\begin{multline*}
\int_{\reels^{2p}} f(\bm_1,\bm_2) \tF(\bm_1,\bm_2) d\bm_1d\bm_2=\\
\sqrt{\frac{|\tbSigma|}{|\bSigma_1|  |\bSigma_2|}} \exp\left\{-\frac12 \left[ \bmu_1^T \bSigma_1^{-1} \bmu_1+ \bmu_2^T \bSigma_2^{-1} \bmu_2 - \tbmu^T \tbSigma^{-1} \tbmu\right]\right\}.
\end{multline*}

\section{Proof of Lemma \ref{lem:proj}}
\label{app:proj}

 The membership function of the projection of fuzzy vector $\GFV(\bm,\bH)$ on $\Theta_1$ is
 \begin{equation}
 \label{eq:phimarg}
 \varphi(\bx_1)=\max_{\bx_2} \exp\left(-\frac12(\bx-\bm)^T \bH (\bx-\bm)\right)= \exp\left(-\frac12 \min_{\bx_2}Z \right),
 \end{equation}
with $Z=(\bx-\bm)^T \bH (\bx-\bm)$. Now,
\begin{subequations}
\label{eq:Z1}
\begin{align}
Z&=(\bx_1-\bm_1, \bx_2-\bm_2) \block{\bH_{11}}{\bH_{12}}{\bH_{21}}{\bH_{22}}\begin{pmatrix}\bx_1-\bm_1\\ \bx_2-\bm_2\end{pmatrix} \\
&=(\bx_1-\bm_1)^T\bH_{11}(\bx_1-\bm_1)+ (\bx_2-\bm_2)^T\bH_{21}(\bx_1-\bm_1)+\\
& \hspace{2cm} (\bx_1-\bm_1)^T\bH_{12}(\bx_2-\bm_2)+ (\bx_2-\bm_2)^T\bH_{22}(\bx_2-\bm_2). \nonumber
\end{align}
\end{subequations}
Using $\bH_{21}=\bH_{12}^T$, the gradient of $Z$ with respect to $\bx_2$ can be written as
\[
\deriv{Z}{\bx_2}=2 \bH_{21}(\bx_1-\bm_1)+ 2 \bH_{22}(\bx_2-\bm_2).
\]
Setting $\deriv{Z}{\bx_2}=0$, and assuming $\bH_{22}$ to be nonsingular, we get
\begin{equation}
\label{eq:x2m2}
(\bx_2-\bm_2)=-\bH_{22}^{-1}\bH_{21}(\bx_1-\bm_1).
\end{equation}
Replacing  $(\bx_2-\bm_2)$ by its expression \eqref{eq:x2m2} in \eqref{eq:Z1} and using \eqref{eq:phimarg}, we finally get
\begin{equation*}
\varphi(\bx_1)= \exp\left(-\frac12(\bx_1-\bm_1)^T \bH_{11}' (\bx_1-\bm_1)\right),
\end{equation*}
with
\[
\bH'_{11}=\bH_{11} - \bH_{12} \bH_{22}^{-1}\bH_{21}.
\]

\end{document}